%% file: main.tex
\documentclass{article}

\usepackage{microtype}
\usepackage{graphicx}
\usepackage{subfigure}
\usepackage{booktabs} 
\input{math_commands.tex}

\usepackage{hyperref}
\usepackage{url}
\usepackage{graphicx}
\usepackage{graphicx}
\usepackage{subfigure}
\usepackage{amsmath}
\usepackage{amsfonts}
\usepackage{enumitem}
\usepackage{amsthm}
\usepackage{amsthm}
\usepackage{xcolor}
\usepackage{xspace}
\usepackage{bbm}
\usepackage{easybmat}
\usepackage{booktabs} 
\usepackage[ruled,vlined]{algorithm2e}

\usepackage{hyperref}
\newtheorem{lemma}{Lemma}
\newtheorem{prop}{Proposition}

\makeatletter
\DeclareRobustCommand\onedot{\futurelet\@let@token\@onedot}
\def\@onedot{\ifx\@let@token.\else.\null\fi\xspace}
\def\eg{\emph{e.g}\onedot} 
\def\ie{\emph{i.e}\onedot}

\makeatother

\usepackage[accepted]{icml2020}

\begin{document}
\theoremstyle{definition}
\twocolumn[
\icmltitle{Latent Variable Modelling with Hyperbolic Normalizing Flows}



\icmlsetsymbol{equal}{*}

\begin{icmlauthorlist}
\icmlauthor{Avishek Joey Bose}{mcgill,mila}
\icmlauthor{Ariella Smofsky}{mcgill,mila}
\icmlauthor{Renjie Liao}{uoft,vector}
\icmlauthor{Prakash Panangaden}{mcgill,mila}
\icmlauthor{William L. Hamilton}{mcgill,mila}
\end{icmlauthorlist}

\icmlaffiliation{mcgill}{McGill University}
\icmlaffiliation{mila}{Mila}
\icmlaffiliation{uoft}{University of Toronto}
\icmlaffiliation{vector}{Vector Institute}

\icmlcorrespondingauthor{Joey Bose}{joey.bose@mail.mcgill.ca}
\icmlcorrespondingauthor{Ariella Smofsky}{ariella.smofsky@mail.mcgill.ca}
\icmlkeywords{Machine Learning, ICML}

\vskip 0.3in
]



\printAffiliationsAndNotice{}  

\begin{abstract}
The choice of approximate posterior distributions plays a central role in stochastic variational inference (SVI). One effective solution is the use of normalizing flows to construct flexible posterior distributions. 
However, a key limitation of existing normalizing flows is that they are restricted to Euclidean space and are ill-equipped to model data with an underlying hierarchical structure.
To address this fundamental limitation, we present the first extension of normalizing flows to hyperbolic spaces. 
We first elevate normalizing flows to hyperbolic spaces using coupling transforms defined on the tangent bundle, termed Tangent Coupling ($\mathcal{TC}$). 
We further introduce Wrapped Hyperboloid Coupling ($\mathcal{W}\mathbb{H}C$), a fully invertible and learnable transformation that explicitly utilizes the geometric structure of hyperbolic spaces, allowing for expressive posteriors while being efficient to sample from. We demonstrate the efficacy of our novel normalizing flow over hyperbolic VAEs and Euclidean normalizing flows. 
Our approach achieves improved performance on density estimation, as well as reconstruction of real-world graph data, which exhibit a hierarchical structure. 
Finally, we show that our approach can be used to power a generative model over hierarchical data using hyperbolic latent variables. 
\end{abstract}

\input{introduction.tex}

\input{background.tex}

\input{method.tex}
\input{experiments.tex}

\input{relatedwork.tex}

\input{conclusion.tex}

\subsection*{Acknowledgements}
 Funding: AJB is supported by an IVADO Excellence Fellowship. RL was supported by Connaught International Scholarship and RBC Fellowship. WLH is supported by a Canada CIFAR AI Chair. This work was also supported by NSERC Discovery Grants held by WLH and PP. In addition the authors would like to thank Chinwei Huang, Maxime Wabartha, Andre Cianflone and Andrea Madotto for helpful feedback on earlier drafts of this work and Kevin Luk, Laurent Dinh and Niky Kamran for helpful technical discussions. The authors would also like to thank the anonymous reviewers for their comments and feedback and Aaron Lou, Derek Lim, and Leo Huang for catching a bug in the code.

\bibliography{bibliography.bib}
\bibliographystyle{icml2020}

\clearpage
\input{appendix.tex}
\end{document}

%% file: math_commands.tex
\usepackage{amsmath,amsfonts,bm}

\newenvironment{proofsketch}{%
  \proof}{\endproof}
  
\newcommand{\RNum}[1]{\uppercase\expandafter{\romannumeral #1\relax}}
\newcommand{\mb}[1]{\mathbf{#1}}
\newcommand{\mbb}[1]{\mathbb{#1}}
\newcommand{\mc}[1]{\mathcal{#1}}

\newcommand{\V}{\mathcal{V}}

\newcommand{\red}{\textcolor{red}}

\newcommand{\xhdr}[1]{{\noindent\bfseries #1}.}
\newcommand{\cut}[1]{}

\newcommand{\removelatexerror}{\let\@latex@error\@gobble}

\def\eqref#1{equation~\ref{#1}}

\def\1{\bm{1}}

\DeclareMathAlphabet{\mathsfit}{\encodingdefault}{\sfdefault}{m}{sl}
\SetMathAlphabet{\mathsfit}{bold}{\encodingdefault}{\sfdefault}{bx}{n}

\DeclareMathOperator{\arccosh}{arccosh}

%% file: introduction.tex
\section{Introduction}
Stochastic variational inference (SVI) methods provide an appealing way of scaling probabilistic modeling to large scale data.
These methods transform the problem of computing an intractable posterior distribution to finding the best approximation within a class of tractable probability distributions \cite{hoffman2013stochastic}.
Using tractable classes of approximate distributions, \eg, mean-field, and Bethe approximations, facilitates efficient inference, at the cost of limiting the expressiveness of the learned posterior. 

In recent years, the power of these SVI methods has been further improved by employing {\em normalizing flows}, which greatly increase the flexibility of the approximate posterior distribution. 
Normalizing flows involve learning a series of invertible transformations, which are used to transform a sample from a simple base distribution to a sample from a richer distribution \cite{rezende2015variational}. 
Indeed, flow-based posteriors enjoy many advantages such as efficient sampling, exact likelihood estimation, and low-variance gradient estimates when the base distribution is reparametrizable, making them ideal for modern machine learning problems.
There have been numerous advances in normalizing flow construction in Euclidean spaces, such as RealNVP \cite{dinh2016density}, B-NAF \cite{huang2018neural,de2019block}, and FFJORD \cite{grathwohl2018ffjord}, to name a few.

\begin{figure}
    \centering
    \vspace{-10pt}
    \includegraphics[width=0.9\linewidth]{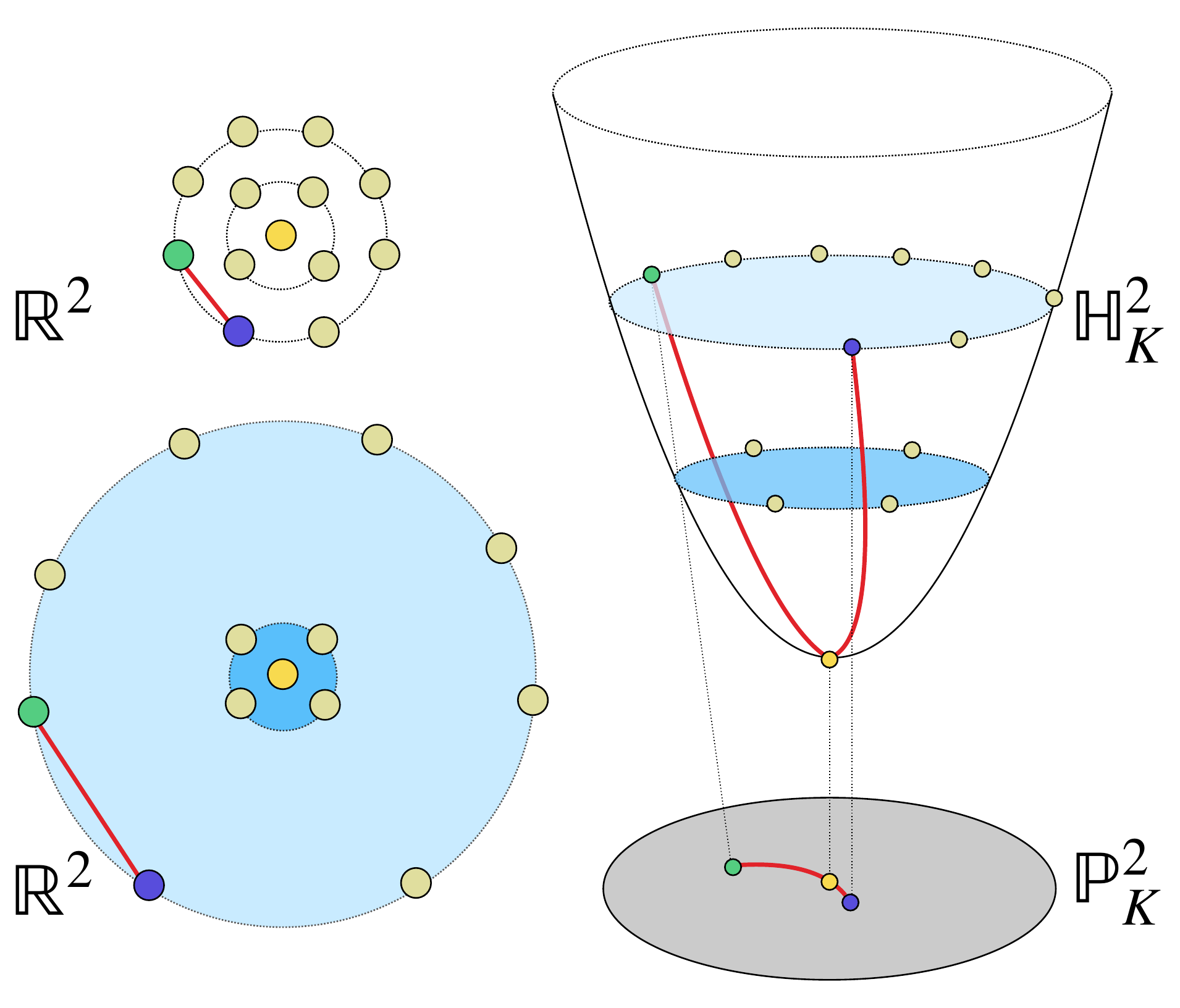}
    \vspace{-10pt}
    \caption{The shortest path between a given pair of node embeddings in $\mathbb{R}^2$ and hyperbolic space as modelled by the Lorentz model $\mathbb{H}^2_K$ and Poincar\'e disk $\mathbb{P}^2_K$. Unlike Euclidean space, distances between points grow exponentially as you move away from the origin in hyperbolic space, and thus the shortest paths between points in hyperbolic space go through a common parent node (i.e., the origin), giving rise to hierarchical and tree-like structure. }
    \vspace{-10pt}
    \label{fig:explanatory_fig_1}
\end{figure}

However, current normalizing flows are restricted to Euclidean space, and as a result, these approaches are ill-equipped to model data with an underlying hierarchical structure. 
Many real-world datasets---such as ontologies, social networks, sentences in natural language, and evolutionary relationships between biological entities in phylogenetics---exhibit rich hierarchical or tree-like structure.
Hierarchical data of this kind can be naturally represented in hyperbolic spaces, \ie, non-Euclidean spaces with constant negative curvature (Figure \ref{fig:explanatory_fig_1}). 
But Euclidean normalizing flows fail to incorporate these structural inductive biases, since Euclidean space cannot embed deep hierarchies without suffering from high distortion \cite{sarkar2011low}. Furthermore, sampling from densities defined on Euclidean space will inevitability generate points that do not lie on the underlying hyperbolic space. 



\xhdr{Present work} 
To address this fundamental limitation, we present the first extension of normalizing flows to hyperbolic spaces. 
Prior works have considered learning models with hyperbolic parameters  \cite{liu2019hyperbolic,nickel2018learning} as well as variational inference with hyperbolic latent variables \cite{nagano2019wrapped,mathieu2019continuous}, but our work represents the first approach to allow flexible density estimation in hyperbolic space. 

To define our normalizing flows we leverage the Lorentz model of hyperbolic geometry and introduce two new forms of coupling, {\em Tangent Coupling} ($\mathcal{TC}$) and {\em Wrapped Hyperboloid Coupling} ($\mathcal{W}\mathbb{H}\mathcal{C}$). These define flexible and invertible transformations capable of transforming sampled points in the hyperbolic space. 
We derive the change of volume associated with these transformations and show that it can be computed efficiently with $\mathcal{O}(n)$ cost, where $n$ is the dimension of the hyperbolic space.
We empirically validate our proposed normalizing flows on structured density estimation, reconstruction, and generation tasks on hierarchical data, highlighting the utility of our proposed approach. \cut{We find that using Wrapped Hyperbolic flows allows for a \red{\%XX} an aggregate absolute improvement over Euclidean flows.}

%% file: background.tex
\section{Background on Hyperbolic Geometry}\label{sec:background}

Within the Riemannian geometry framework, hyperbolic spaces are manifolds with constant negative curvature $K$ and are of particular interest for embedding hierarchical structures. There are multiple models of $n$-dimensional hyperbolic space, such as the hyperboloid $\mathbb{H}_K^n$, also known as the Lorentz model, or the Poincar\'e ball $\mathbb{P}_K^n$. 
Figure \ref{fig:explanatory_fig_1} illustrates some key properties of $\mathbb{H}_K^2$ and $\mathbb{P}_K^2$, highlighting how distances grow exponentially as you move away from the origin and how the shortest paths between distant points tend to go through a common parent (\textit{i.e.}, the origin), giving rise to a hierarchical or tree-like structure. 
In the next section, we briefly review the Lorentz model of hyperbolic geometry. We are not assuming a background in Riemannian geometry, though Appendix \ref{riemannian_geometry_appendix} and \citet{Ratcliffe94} are of use to the interested reader. Henceforth, for notational clarity, we use boldface font to denote points on the hyperboloid manifold.


\subsection{Lorentz Model of Hyperbolic Geometry}
An $n$-dimensional hyperbolic space, $\mathbb{H}^n_K$, is the unique, complete, simply-connected $n$-dimensional Riemannian manifold of constant negative curvature, $K$.  For our purposes, the Lorentz model is the most convenient representation of hyperbolic space, since it is equipped with relatively simple explicit formulas and useful numerical stability properties \cite{nickel2018learning}. We choose the 2D Poincar\'e disk $\mathbb{P}^2_1$ to visualize hyperbolic space because of its conformal mapping to the unit disk. The Lorentz model embeds hyperbolic space $\mathbb{H}^n_K$ within the $n+1$-dimensional Minkowski space, defined as the manifold $\mathbb{R}^{n+1}$ equipped with the following inner product:
\begin{equation}\label{eq:lorentzmetric}
    \langle \textbf{x}, \textbf{y} \rangle_{\mathcal{L}} := -x_0y_0 + x_1y_1 + \dots + x_ny_n,
\end{equation}
which has the type 
$\langle \cdot, \cdot \rangle_{\mathcal{L}}: \mathbb{R}^{n+1} \times \mathbb{R}^{n+1} \to \mathbb{R}$.
It is common to denote this space as $\mathbb{R}^{1,n}$ to emphasize the distinct role of the zeroth coordinate.
In the Lorentz model, we model hyperbolic space as the (upper sheet of) the hyperboloid embedded in Minkowski space. It is a remarkable fact that though the Lorentzian metric (Eq.~\ref{eq:lorentzmetric}) is indefinite,
the induced Riemannian metric $g_{\mb{x}}$ on the unit hyperboloid is positive definite \cite{Ratcliffe94}. The $n$-Hyperbolic space with constant negative curvature $K$ with origin $\textbf{o} = (1/K, 0, \dots, 0)$, is a Riemannian manifold $(\mathbb{H}^{n}_K,g_{\mb{x}})$ where
\begin{equation*}
    \mathbb{H}^{n}_K := \{x \in \mathbb{R}^{n+1}:  \langle \textbf{x}. \textbf{x} \rangle_{\mathcal{L}} = 1/K, \ x_0 > 0, \ K<0 \}.
\end{equation*}


Equipped with this, the induced distance between two points $(\mb{x},\mb{y})$ in $\mathbb{H}^{n}_K$ is given by
\begin{equation}
    d(\textbf{x},\textbf{y})_{\mathcal{L}} := \frac{1}{\sqrt{-K}}\arccosh(-K \langle \textbf{x}, \textbf{y} \rangle_{\mathcal{L}}).
\end{equation}

The tangent space to the hyperboloid at the point $\textbf{p} \in \mathbb{H}_K^n$ can also be described as an embedded subspace of $\mathbb{R}^{1,n}$.  
It is given by the set of points satisfying the orthogonality relation with respect to the Minkowski inner product,\footnote{It is also equivalently known as the Lorentz inner product.}
\begin{equation}
   \mathcal{T}_{\textbf{p}}\mathbb{H}^{n}_K := \{ u: \langle u, \textbf{p} \rangle_{\mathcal{L}} = 0\}
\end{equation}
Of special interest are vectors in the tangent space at the origin of $\mathbb{H}^{n}_K$ whose norm under the Minkowski inner product is equivalent to the conventional Euclidean norm. That is $v \in \mathcal{T}_{\textbf{o}}\mathbb{H}^{n}_K$ is a vector such that $v_0 = 0$ and $||\textbf{v}||_{\mathcal{L}} := \sqrt{\langle \textbf{v}, \textbf{v}
\rangle_{\mathcal{L}}} = ||\textbf{v}||_2$.  Thus \emph{at the origin} the
partial derivatives with respect to the ambient coordinates, $\mathbb{R}^{n+1}$, define the
covariant derivative.  

\xhdr{Projections}
Starting from the extrinsic view by which we consider $\mathbb{R}^{n+1} \supset \mathbb{H}^{n}_K$, we may project any vector $x \in \mathbb{R}^{n+1}$ on to the hyperboloid using the shortest Euclidean distance:
\begin{equation}
    \textnormal{proj}_{\mathbb{H}^{n}_K}(x) = \frac{x}{\sqrt{-K}||x||_{\mathcal{L}}}.
\end{equation}
Furthermore, by definition a point on the hyperboloid satisfies $\langle \textbf{x}, \textbf{x} \rangle_{\mathcal{L}} = 1/K$ and thus when provided with $n$ coordinates $\hat{x} = (x_1, \dots, x_{n})$ we can always determine the missing coordinate to get a point on $\mathbb{H}^n_K$:
\begin{equation}
    \label{eqn:hyperboloid_projection}
    x_0 = \sqrt{||\hat{x}||^2_2 + \frac{1}{K}}.
\end{equation}

\xhdr{Exponential Map}
The exponential map takes a vector, $v$, in the tangent space of a point $\textbf{x} \in \mathbb{H}^{n}_K$ to a point on the manifold---i.e., $\textbf{y} = \textnormal{exp}^K_\textbf{x}(v): \mathcal{T}_{\textbf{x}} \mathbb{H}^{n}_K \to \mathbb{H}^{n}_K$ by moving a unit length along the \textit{geodesic}, $\gamma$ (straightest parametric curve), uniquely defined by $\gamma(0) = \textbf{x}$ with direction $\gamma '(0)= v$. The closed form expression for the exponential map is then given by
\begin{equation}
    \label{eqn:exp_map}
    \textnormal{exp}^K_{\textbf{x}} (v) = \cosh \Big(\frac{||v||_{\mathcal{L}}}{R} \Big)\textbf{x} +  \sinh \Big(\frac{||v||_{\mathcal{L}}}{R} \Big)\frac{Rv}{||{v}||_{\mathcal{L}}},
\end{equation}
where we used the \textit{generalized radius} $R = 1/\sqrt{-K}$ in place of the curvature.

\xhdr{Logarithmic Map}
As the inverse of the exponential map, the logarithmic map takes a point, $\textbf{y}$, on the manifold back to the tangent space of another point $\textbf{x}$ also on the manifold. In the Lorentz model this is defined as
\begin{equation}
    \label{eqn:log_map}
    \log^K_{\textbf{x}}{\textbf{y}} = \frac{\arccosh(\alpha)}{\sqrt{\alpha^2 - 1}}(\textbf{y} - \alpha \textbf{x}),
\end{equation}
where $\alpha = K\langle \textbf{x}, \textbf{y} \rangle_{\mathcal{L}}$.

\xhdr{Parallel Transport}
The parallel transport for two points $\textbf{x},\textbf{y} \in \mathbb{H}^{n}_K$ is a map that carries the vectors in $v \in \mathcal{T}_{\textbf{x}}\mathbb{H}^{n}_K$ to corresponding vectors at $v' \in \mathcal{T}_{\textbf{y}}\mathbb{H}^{n}_K$ along the geodesic. That is vectors are connected between the two tangent spaces such that the covariant derivative is unchanged. Parallel transport is a map that preserves the metric, \ie, $\langle \textnormal{PT}^K_{\textbf{x} \to \textbf{y}}(v), \textnormal{PT}^K_{\textbf{x} \to \textbf{y}}(v') \rangle_{\mathcal{L}} = \langle v, v' \rangle_{\mathcal{L}}$ and in the Lorentz model is given by
\begin{align}
    \label{eqn:parallel_transport}
    \textnormal{PT}^K_{\textbf{x} \to \textbf{y}}(v) & = v - \frac{ \langle \log^K_{\textbf{x}}(\textbf{y}), v \rangle_{\mathcal{L}}}{d(\textbf{x},\textbf{y})_{\mathcal{L}}} (\log^K_{\textbf{x}}(\textbf{y})+ \log^K_{\textbf{y}}(\textbf{x})) \nonumber \\
    & = v + \frac{ \langle \textbf{y}, v \rangle_{\mathcal{L}}}{R^2 - \langle \textbf{x}, \textbf{y} \rangle_{\mathcal{L}}} (\textbf{x}+ \textbf{y}),
\end{align}
where $\alpha$ is as defined above. Another useful property is that the inverse parallel transport simply carries the vectors back along the geodesic and is simply defined as $(\textnormal{PT}^{K}_{\textbf{x} \to \textbf{y}}(v))^{-1} = \textnormal{PT}^K_{\textbf{y} \to \textbf{x}}(v)$.

\subsection{Probability Distributions on Hyperbolic Spaces}\label{sec:hyperprobs}

Probability distributions can be defined on Riemannian manifolds, which include $\mathbb{H}^n_K$ as a special case. One transforms the infinitesimal volume element on the manifold to the corresponding volume element in $\mathbb{R}^n$ as defined by the co-ordinate charts. In particular, given the Riemannian manifold $\mathcal{M}(\textbf{z})$ and its metric $g_\mb{z}$, we have $\int p(\textbf{z}) d\mathcal{M} (\textbf{z}) = \int p(\textbf{z})\sqrt{|g_\mb{z}|}d \textbf{z}$, where $d \textbf{z}$ is the Lebesgue measure. We now briefly survey three distinct generalizations of the normal distribution to Riemannian manifolds.\cut{ following ~\cite{mathieu2019continuous,skopek2019mixed}.}

\xhdr{Riemannian Normal}
The first is the Riemannian normal~\cite{pennec2006intrinsic, said2014new}, which is derived from maximizing the entropy given a Fr\'echet mean $\mu$ and a dispersion parameter $\sigma$.
Specifically, we have $\mathcal{N}_{\mathcal{M}}(\textbf{z} \vert \mu, \sigma^{2}) = \frac{1}{Z} \exp \left( - d_{\mathcal{M}}(\mu, \textbf{z})^2 / 2 \sigma^{2} \right)$, where $d_{\mathcal{M}}$ is the \textit{induced distance} and $Z$ is the normalization constant \cite{said2014new, mathieu2019continuous}.

\xhdr{Restricted Normal}
One can also restrict sampled points from the normal distribution in the ambient space to the manifold.
One example is the Von Mises distribution on the unit circle and its generalized version, \ie, Von Mises-Fisher distribution on the hypersphere~\cite{davidson2018hyperspherical}.

\xhdr{Wrapped Normal}
Finally, we can define a wrapped normal distribution~\cite{falorsi2019reparameterizing,nagano2019wrapped}, which is obtained by (1) sampling from $\mathcal{N}(0,I)$ and then transforming it to a point $v \in \mathcal{T}_\textbf{o}\mathbb{H}_K^n$ by concatenating $0$ as the zeroth coordinate; (2) parallel transporting the sample $v$ from the tangent space at $\textbf{o}$ to the tangent space of another point $\boldsymbol{\mu} $ on the manifold to obtain $u$; (3) mapping $u$ from the tangent space to the manifold using the exponential map at $\boldsymbol{\mu}$.
Sampling from such a distribution is straightforward and the probability density can be obtained via the change of variable formula,
\begin{align}
    \log p(\textbf{z}) = \log p(v) - (n-1) \log \left(\frac{\sinh{ ( \Vert u \Vert_{\mathcal{L}} ) }}{ \Vert u \Vert_{\mathcal{L}} } \right),
\end{align}
where $p(\textbf{z})$ is the wrapped normal distribution and $p(v)$ is the normal distribution in the tangent space of $\textbf{o}$.

%% file: method.tex
\section{Normalizing Flows on Hyperbolic Spaces}
We seek to define flexible and learnable distributions on $\mathbb{H}^n_K$, which will allow us to learn rich approximate posterior distributions for hierarchical data.
To do so, we design a class of invertible parametric hyperbolic functions, $f_i: \mathbb{H}^n_K \to \mathbb{H}^n_K$.
A sample from the approximate posterior can then be obtained by first sampling from a simple base distribution $\mb{z}_0 \sim p(\mb{z})$ defined on $\mathbb{H}^n_K$ and then applying a composition of  functions $f_{i\in[j]}$ from this class: $\mb{z}_j = f_j \circ f_{j-1} \circ \dots \circ f_1(\mb{z}_0)$.

In order to ensure effective and tractable learning, the class of functions $f_i$ must satisfy three key desiderata:
\begin{enumerate}[itemsep=0pt, parsep=0pt, topsep=0pt]
    \item Each function $f_i$ must be invertible. 
    \item We must be able to efficiently sample from the final distribution, $\mb{z}_j = f_j \circ f_{j-1} \circ \dots \circ f_1(\mb{z}_0)$. 
    \item We must be able to efficiently compute the associated change in volume (\textit{i.e.}, the Jacobian determinant) of the overall transformation.
\end{enumerate}
 Given these requirements, the final transformed distribution is given by the change of variables formula:
\begin{equation}
    \log p(\mb{z}_j) = \log p(\mb{z}_0) - \sum_{i=1}^k\log \textnormal{det}\Big|\frac{\partial f_j}{\partial z_{j-1}} \Big|.
    \vspace{-5pt}
\end{equation}

Functions satisfying desiderata 1-3 in Euclidean space are often termed {\em normalizing flows} (Appendix \ref{normalizing_flow_appendix}), and our work extends this idea to hyperbolic spaces. In the following sections, we describe two flows of increasing complexity: Tangent Coupling ($\mathcal{T}C$) and Wrapped Hyperboloid Coupling ($\mathcal{W}HC$). The first approach lifts a standard Euclidean flow to the tangent space at the origin of the hyperboloid.
The second approach modifies the flow to explicitly utilize hyperbolic geometry. Figure \ref{fig:density_estimation} illustrates synthetic densities as learned by our approach on $\mathbb{P}^2_1$. 

\subsection{Tangent Coupling}
Similar to the Wrapped Normal distribution (Section \ref{sec:hyperprobs}), one strategy to define a normalizing flow on the hyperboloid is to use the tangent space at the origin. That is, we first sample a point from our base distribution---which we define to be a Wrapped Normal---and use the logarithmic map at the origin to transport it to the corresponding tangent space. Once we arrive at the tangent space we are free to apply any Euclidean flow before finally projecting back to the manifold using the exponential map. 
This approach leverages the fact that the tangent bundle of a hyperbolic manifold has a well-defined vector space structure, allowing affine transformations and other operations that are ill-defined on the manifold itself. 

Following this idea, we build upon one of the earliest and most well-studied flows: the RealNVP flow \cite{dinh2016density}. At its core, the RealNVP flow uses a computationally symmetric transformation (affine coupling layer), which has the benefit of being fast to evaluate and invert due to its lower triangular Jacobian, whose determinant is cheap to compute. Operationally, the coupling layer is implemented using a binary mask, and partitions some input $\tilde{x}$ into two sets, where the first set, $\tilde{x}_1:=\tilde{x}_{1:d}$, is transformed elementwise independently of other dimensions. The second set, $\tilde{x}_2:=\tilde{x}_{d+1:n}$, is also transformed elementwise but in a way that depends on the first set (see Appendix \ref{Euclidean_RealNVP_appendix} for more details). Since all coupling layer operations occur at $\mathcal{T}_{\textbf{o}}\mathbb{H}^n_K$ we term this form of coupling as Tangent Coupling ($\mathcal{T}C$). 

Thus, the overall transformation due to one layer of our $\mathcal{T}C$ flow is a composition of a logarithmic map, affine coupling defined on $\mathcal{T}_{\textbf{o}}\mathbb{H}^n_k$, and an exponential map:
\begin{align}
    \label{eq:tangent_coupling}
     \tilde{f}^{\mathcal{T}C}(\tilde{x}) &=
     \begin{cases}
     \tilde{z}_{1} &= \tilde{x}_{1} \\
     \tilde{z}_{2} &= \tilde{x}_{2} \odot \sigma(s(\tilde{x}_{1})) + t(\tilde{x_1}) \nonumber
     \end{cases}\\
    f^{\mathcal{T}C}(\mb{x}) &= \textnormal{exp}_{\textbf{o}}^K(\tilde{f}^{\mathcal{T}C}(\log_{\textbf{o}}^K(\textbf{x}))),
\end{align}

where $\tilde{x} = \log_{\textbf{o}}^K(\textbf{x})$ is a point on $\mathcal{T}_{\textbf{o}}\mathbb{H}^n_K$, and $\sigma$ is a pointwise non-linearity such as the exponential function. Functions $s$ and $t$ are parameterized scale and translation functions implemented as neural nets from $\mathcal{T}_{\textbf{o}}\mathbb{H}^{d}_K \to \mathcal{T}_{\textbf{o}}\mathbb{H}^{n-d}_K$.
One important detail is that arbitrary operations on a tangent vector $v \in \mathcal{T}_{\textbf{o}}\mathbb{H}^n_K$ may transport the resultant vector outside the tangent space, hampering subsequent operations.
To avoid this we can keep the first dimension fixed at $v_0 = 0$ to ensure we remain in $\mathcal{T}_{\textbf{o}}\mathbb{H}^{n}_K$.

\begin{figure}[t!]
     \centering
     \includegraphics[width=0.95\linewidth]{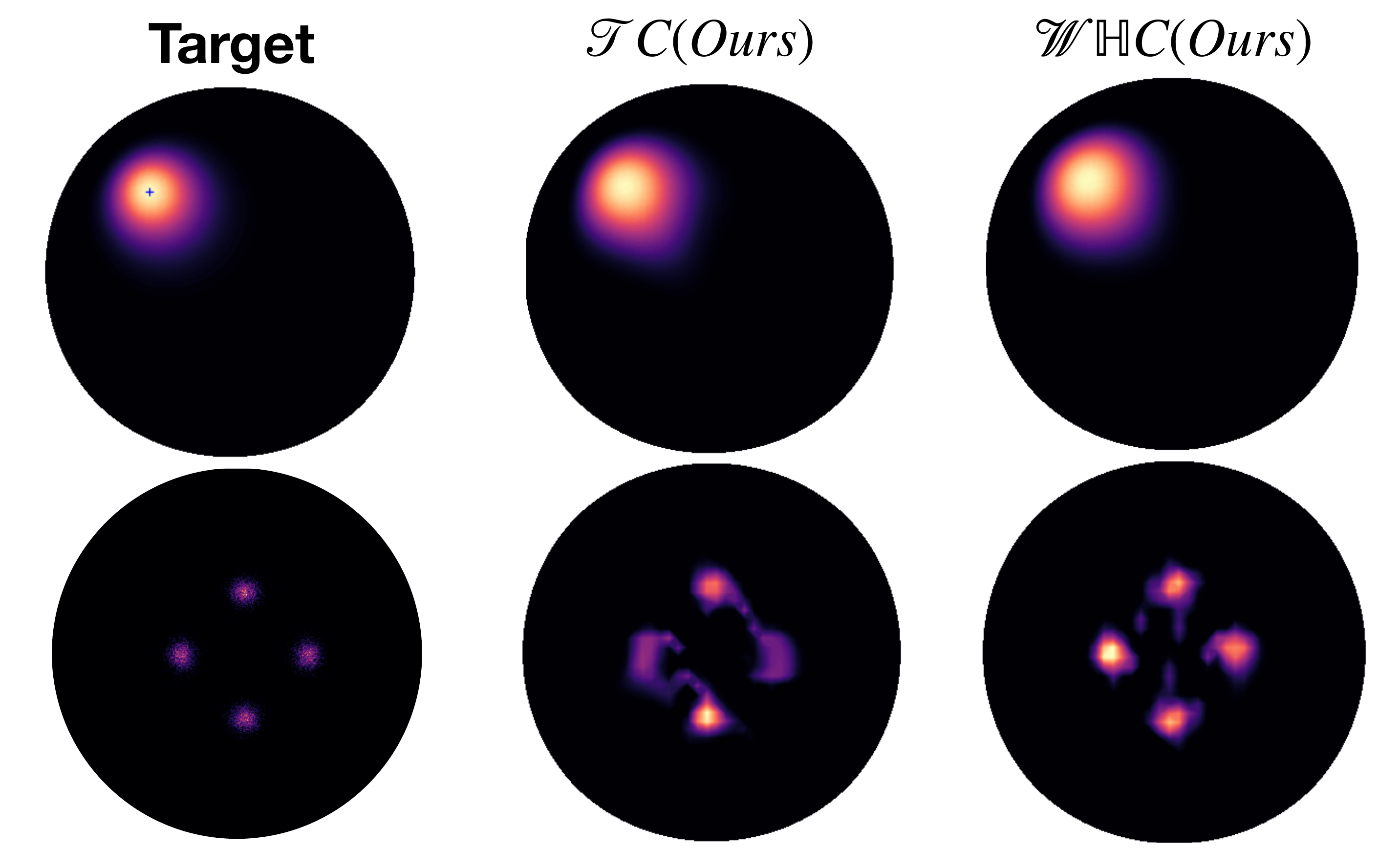}
     \vspace{-10pt}
     \caption{Comparison of density estimation in hyperbolic space for 2D wrapped Gaussian (WG) and mixture of wrapped gaussian (MWG) on $\mathbb{P}^2_1$. Densities are visualized in the Poincar\'e disk. Additional qualitative results can be found in Appendix \ref{additional_results}.}
     \vspace{-10pt}
     \label{fig:density_estimation}
 \end{figure}

Similar to the Euclidean RealNVP, we need an efficient expression for the Jacobian determinant of $f^{\mathcal{T}C}$.
\begin{prop}
The Jacobian determinant of a single $\mathcal{T}C$ layer in \eqref{eq:tangent_coupling} is:
\begin{align}
     \left|\textnormal{det}\Big (\frac{\partial \textbf{y}}{\partial \textbf{x}}\Big)\right| &= \Big(\frac{R\sinh(\frac{||z||_{\mathcal{L}}}{R})}{||z||_{\mathcal{L}}}\Big)^{n-1} \times \prod_{i=d+1}^n\sigma(s(\tilde{x}_1))_i  \nonumber\\
     & \times \Big(\frac{R\sinh(\frac{||\log^K_{\textbf{o}}(\textbf{x})||_{\mathcal{L}}}{R})}{||\log^K_{\textbf{o}}(\textbf{x})||_{\mathcal{L}}}\Big)^{1-n}
\end{align}
where, $ \mb{z} =  \tilde{f}^{\mathcal{T}C}(\tilde{x})$ and $\tilde{f}^{\mathcal{T}C}$ is as defined above.
\end{prop}
\begin{proofsketch}
Here we only provide a sketch of the proof and details can be found in Appendix \ref{tangent_coupling_proof_appendix}. First, observe  that the overall transformation  is a valid composition of functions: $\mb{y} := \textnormal{exp}_{\textbf{o}}^K \circ \tilde{f}^{\mathcal{T}C} \circ \log_{\textbf{o}}^K(\textbf{x})$. Thus, the overall determinant can be computed by chain rule and the identity,  $\textnormal{det}\Big (\frac{\partial \textbf{y}}{\partial \textbf{x}}\Big) =  \textnormal{det}\Big (\frac{\partial \textnormal{exp}_{\textbf{o}}^K(z)}{\partial z}\Big) \cdot \textnormal{det}\Big (\frac{\partial f(\tilde{x})}{\partial \tilde{x}}\Big) \cdot  \textnormal{det}\Big (\frac{\partial \log_{\textbf{o}}^K(\textbf{x})}{\partial \textbf{x}}\Big)$. Tackling each function in the composition individually, $\textnormal{det}\Big (\frac{\partial \textnormal{exp}_{\textbf{o}}^K(z)}{\partial z}\Big) = \Big(\frac{R\sinh(\frac{||z||_{\mathcal{L}}}{R})}{||z||_{\mathcal{L}}}\Big)^{n-1}$ as derived in \citet{skopek2019mixed}. As the logarithmic map is the inverse of the exponential map the Jacobian determinant is simply the inverse of the determinant of the exponential map, which gives the $\textnormal{det}\Big (\frac{\partial \log_{\textbf{o}}^K(\textbf{x})}{\partial \textbf{x}}\Big)$ term. 
For the middle term, we must calculate the directional derivative of $\tilde{f}^{\mathcal{T}C}$ in an orthonormal basis w.r.t. the Lorentz inner product, of $\mathcal{T}_{\textbf{o}}\mathbb{H}^{n}_K$. Since the standard Euclidean basis vectors $e_1, ..., e_n$ are also a basis for $\mathcal{T}_{\textbf{o}}\mathbb{H}^{n}_K$, the Jacobian determinant $\textnormal{det}\Big (\frac{\partial f(\tilde{x})}{\partial \tilde{x}}\Big)$ simplifies to that of the RealNVP flow, which is lower triangluar and is thus efficiently computable in $\mathcal{O}(n)$ time.

\end{proofsketch}
It is remarkable that the middle term in Proposition 1 is precisely the same change in volume associated with affine coupling in RealNVP.
The change in volume due to the hyperbolic space only manifests itself through the exponential and logarithmic maps, each of which can be computed in $\mathcal{O}(n)$ cost. Thus, the overall cost is only slightly larger than the regular Euclidean RealNVP, but still $\mathcal{O}(n)$.

\subsection{Wrapped Hyperboloid Coupling}
\begin{figure}[ht]
    \centering
    \includegraphics[width=0.95\linewidth]{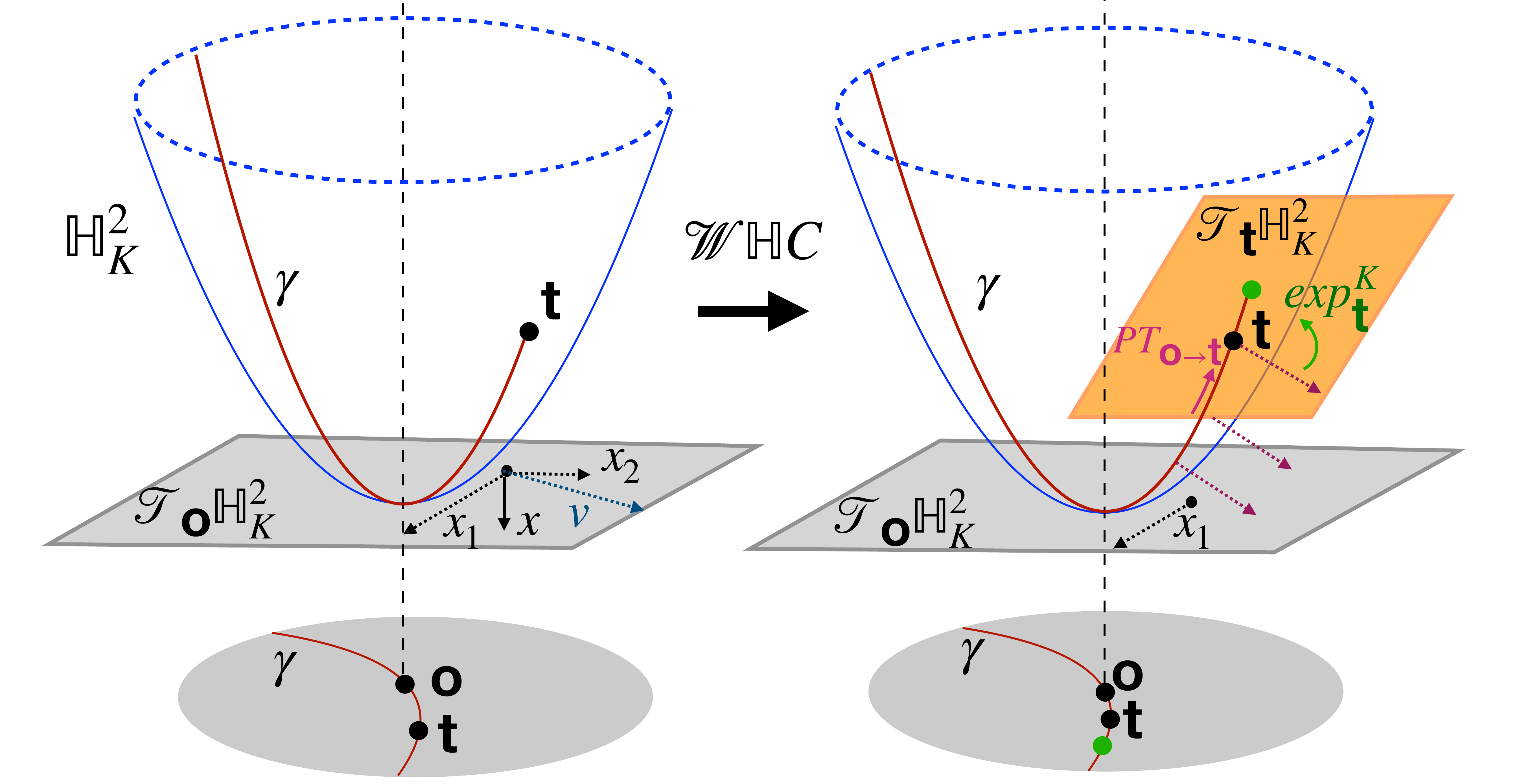}
    \caption{Wrapped Hyperbolic Coupling. The left figure depicts a partitioned input point $\tilde{x}_1:=\tilde{x}_{1:d}$ and $\tilde{x}_2:=\tilde{x}_{d+1:n}$ prior to parallel transport. The right figure depicts the $\tilde{x}_2$ vector after it is transformed, parallel transported, and projected to $\mathbb{H}^n_K$.}
    \label{fig:whc_architecture_diagram}
\end{figure}

\label{wrapped_hyerboloid_coupling_section}
The hyperbolic normalizing flow with $\mathcal{T}C$ layers discussed above operates purely in the tangent space at the origin.
This simplifies the computation of the Jacobian determinant, but anchoring the flow at the origin may hinder its expressive power and its ability to leverage disparate regions of the manifold. 
In this section, we remedy this shortcoming with a new hyperbolic flow that performs translations between tangent spaces via parallel transport. 

We term this transformation {\em Wrapped Hyperboloid Coupling} ($\mathcal{W}\mathbb{H}C$). 
As with the $\mathcal{T}C$ layer, it is a fully invertible transformation $f^{\mathcal{W}\mathbb{H}C}: \mathbb{H}^n_k \to \mathbb{H}^n_k$ with a tractable analytic form for the Jacobian determinant. 
To define a $\mathcal{W}\mathbb{H}C$ layer we first use the logarithmic map at the origin to transport a point to the tangent space. We employ the coupling strategy previously discussed and partition our input vector into two components: $\tilde{x}_1:=\tilde{x}_{1:d}$ and $\tilde{x}_2:=\tilde{x}_{d+1:n}$. Let $\tilde{x} = \log_{\textbf{o}}^K(\mb{x})$ be the point on $\mathcal{T}_{\textbf{o}}\mathbb{H}^n_K$ after the logarithmic map. 
The remainder of the $\mathcal{W}\mathbb{H}C$ layer can be defined as follows;
\begin{align}
\label{wrapped_hyperboloid_coupling_eqn}
\tilde{f}^{\mathcal{W}\mathbb{H}C}(\tilde{x})&=
     \begin{cases}
     \tilde{z}_{1} &= \tilde{x}_{1}  \\
     \tilde{z}_{2} &= \log_{\textbf{o}}^K\Big( \textnormal{exp}_{t(\tilde{x}_{1})}^K\big(\textnormal{PT}_{\textbf{o}\to t(\tilde{x}_{1}) }(v)\big)\Big) \nonumber
    \end{cases}\\
    v &= \tilde{x}_{2} \odot \sigma(s(\tilde{x}_{1})) \nonumber \\
    f^{\mathcal{W}\mathbb{H}C}(\mb{x}) &=  \textnormal{exp}_{\textbf{o}}^K(\tilde{f}^{\mathcal{W}\mathbb{H}C}(\log_{\textbf{o}}^K(\mb{x}))).
\end{align}

Functions $s: \mathcal{T}_{\textbf{o}}\mathbb{H}^{d}_k \to \mathcal{T}_{\textbf{o}}\mathbb{H}^{n-d}_k$ and $t:\mathcal{T}_{\textbf{o}}\mathbb{H}^{d}_k \to \mathbb{H}^n_k$ are taken to be arbitrary neural nets, but the role of $t$ when compared to $\mathcal{T}C$ is vastly different. In particular, the generalization of translation on Riemannian manifolds can be viewed as parallel transport to a different tangent space. Consequently, in Eq.~\ref{wrapped_hyperboloid_coupling_eqn}, the function $t$ predicts a point on the manifold that we wish to parallel transport to.
This greatly increases the flexibility as we are no longer confined to the tangent space at the origin. The logarithmic map is then used to ensure that both $\tilde{z}_1$ and $\tilde{z}_2$ are in the same tangent space before the final exponential map that projects the point to the manifold.

One important consideration in the construction of $t$ is that it should only parallel transport functions of $\tilde{x}_2$. However, the output of $t$ is a point on $\mathbb{H}^n_k$ and without care this can involve elements in $\tilde{x}_1$. To prevent such a scenario we construct the output of $t = [t_0, 0, \dots , 0, t_{d+1}, \dots , t_{n}]$ where elements $t_{d+1:n}$ are used to determine the value of $t_0$ using Eq. \ref{eqn:hyperboloid_projection}, such that it is a point on the manifold and every remaining index is set to zero. Such a construction ensures that only components of any function of $\tilde{x}_2$ are parallel transported as desired. Figure \ref{fig:whc_architecture_diagram} illustrates the transformation performed by the $\mathcal{W}HC$ layer.

\xhdr{Inverse of $\mathcal{W}\mathbb{H}C$}
To invert the flow it is sufficient to show that argument to the final exponential map at the origin itself is invertible. Furthermore, note that $\tilde{x}_1$ undergoes an identity mapping and is trivially invertible. Thus, we need to show that the second partition is invertible, \textit{i.e.} that the following transformation is invertible:
\begin{equation}
     \tilde{z}_{2} = \log_{\textbf{o}}^K\Big( \textnormal{exp}_{t(\tilde{x}_{1})}^K\big(\textnormal{PT}_{\textbf{o}\to t(\tilde{x}_{1}) }(v)\big)\Big).
\end{equation}
As discussed in Section \ref{sec:background}, the parallel transport, exponential map, and logarithmic map all have well-defined inverses with closed forms. Thus, the overall transformation is invertible in closed form:
\begin{align*}
     \begin{cases}
     \tilde{x}_{1} &= \tilde{z}_{1} \\
     \tilde{x}_{2} &= \Big( \textnormal{PT}_{t(\tilde{z}_{1}) \to \textbf{o} }(\log_{t(\tilde{z}_{1})}^K( \textnormal{exp}_{\textbf{o}}^K(\tilde{z}_{2}))\Big) \odot \sigma(s(\tilde{z}_{1}))^{-1} \\
    \end{cases}
\end{align*}

\xhdr{Properties of $\mathcal{W}\mathbb{H}C$}
To compute the Jacobian determinant of the full transformation in Eq. \ref{wrapped_hyperboloid_coupling_eqn} we proceed by analyzing the effect of $\mathcal{W}\mathbb{H}C$ on valid orthonormal bases w.r.t. the Lorentz inner product for the tangent space at the origin. We state our main result here and provide a sketch of the proof, while the entire proof can be found in Appendix \ref{wrapped_coupling_proof_appendix}. 
\begin{prop}
The Jacobian determinant of the function $\tilde{f}^{\mathcal{W}\mathbb{H}C}$ in \eqref{wrapped_hyperboloid_coupling_eqn} is:
\begin{multline}
  \left|\textnormal{det}\left(\frac{\partial\mb{y}}{\partial \mb{x}}\right)\right| = \prod_{i=d+1}^n\sigma(s(\tilde{x}_1))_i \times \Big(\frac{R \sinh(\frac{||q||_{\mathcal{L}}}{R})}{||q||_{\mathcal{L}}}\Big)^{l} \\
  \times \Big(\frac{R\sinh(\frac{||\log_{\textbf{o}}^K(\hat{\textbf{q}})||_{\mathcal{L}}}{R})}{||\log_{\textbf{o}}^K(q)||_{\mathcal{L}}}\Big)^{-l} \ \times  \Big(\frac{R\sinh(\frac{||\tilde{z}||_{\mathcal{L}}}{R})}{||\tilde{z}||_{\mathcal{L}}}\Big)^{n-1}
 \  \\ \times \Big(\frac{R\sinh(\frac{||\log_{\textbf{o}}^K(\textbf{x})||_{\mathcal{L}}}{R})}{||\log_{\textbf{o}}^K(\textbf{x})||_{\mathcal{L}}}\Big)^{1-n},
\end{multline}
where $\tilde{z} = \textnormal{concat}(\tilde{z}_1, \tilde{z}_2)$, the constant $l = n-d$, $\sigma$ is a non-linearity, $q = \textnormal{PT}_{\textbf{o}\to t(\tilde{x}_1) }(v)$ and $\hat{\textbf{q}} = \textnormal{exp}_{\textbf{t}}^K(q)$.
\end{prop}
\begin{proofsketch}
We first note that the exponential and logarithmic maps applied at the beginning and end of the $\mathcal{W}\mathbb{H}C$ can be dealt with by appealing to the chain rule and the known Jacobian determinants for these functions as used in Proposition 1.  
Thus, what remains is the following term: $\left|\textrm{det}\left(\frac{\partial z}{\partial \tilde{x}}\right)\right|$. To evaluate this term we rely on the following Lemma.
\begin{lemma}
Let $h : \mc{T}_\mb{o}\mbb{H}^n_k \rightarrow \mc{T}_\mb{o}\mbb{H}^n_k$ be a function defined as:
\begin{equation}\label{eq:g2}
h(\tilde{x}) = z = \textnormal{concat}(\tilde{z_1}, \tilde{z_2}).
\end{equation}
Now, define a function $h^* : \mc{T}_\mb{o}\mbb{H}^{n-d} \rightarrow \mc{T}_\mb{o}\mbb{H}^{n-d}$ which acts on the subspace of $\mc{T}_\mb{o}\mbb{H}^{n-d}$ corresponding to the standard basis elements $e_{d+1}, ..., e_n$ as
\begin{equation}\label{eq:gstar2}
h^*(\tilde{x}_{2}) =   \log_{\mb{o}_{2}}^K\Big( \textnormal{exp}_{\mb{t}_{2}}^K\big(\textnormal{PT}_{\mb{o}_{2} \to \mb{t}_{2}}(v)\big)\Big),
\end{equation}
where $\tilde{x}_{2}$ denotes the portion of the vector $\tilde{x}$ corresponding to the standard basis elements $e_{d+1}, ..., e_n$ and $s$ and $t$ are constants (which depend on $\tilde{x}_{1}$).
In Equation \eqref{eq:gstar2}, we use $\mb{o}_{2} \in \mbb{H}^{n-d}$ to denote the vector corresponding to only the dimensions $d+1, ..., n$ and similarly for $\mb{t}_{2}$.
Then we have that
\begin{equation}
    \left|\textnormal{det}\left(\frac{\partial z}{\partial \tilde{x}}\right)\right| =    \left|\textnormal{det}\left(\frac{\partial h^*(\tilde{x}_{d+1:n})}{\partial \tilde{x}_{d+1:n})}\right)\right|.
\end{equation}
\end{lemma}
The proof for Lemma 1 is provided in Appendix \ref{wrapped_coupling_proof_appendix}. Using Lemma 1, and the fact that $|\textnormal{det}(\textnormal{PT}_{\textbf{u}\to\textnormal{t}}(v))| = 1$ \cite{nagano2019wrapped} we are left with another composition of functions but on the subspace $\mathcal{T}_{\textbf{o}}\mathbb{H}^{n-d}$. The Jacobian determinant for these functions, are simply that of the logarithmic map, exponential map and the argument to the parallel transport which can be easily computed as $\prod_{i=d+1}^n \sigma(s(\tilde{x}_1))$. 
\end{proofsketch}
The cost of computing the change in volume for one $\mathcal{W}\mathbb{H}C$ layer is $\mathcal{O}(n)$ which is the same as a $\mathcal{T}C$ layer plus the added cost of the two new maps that operate on the lower subspace of basis elements.

%% file: experiments.tex
\section{Experiments}
\label{experiments}
We evaluate our $\mathcal{T}C$-flow and $\mathcal{W}\mathbb{H}C$-flow on three tasks:  structured density estimation, graph reconstruction, and graph generation.\footnote{\url{https://github.com/joeybose/HyperbolicNF}} Throughout our experiments, we rely on three main baselines. 
In Euclidean space, we use Gaussian latent variables and affine coupling flows \cite{dinh2016density}, denoted $\mathcal{N}$ and $\mathcal{N}C$, respectively. In the Lorentz model, we use Wrapped Normal latent variables, $\mathbb{H}$-VAE, as an analogous baseline \cite{nagano2019wrapped}. Since all model parameters are defined on Euclidean tangent spaces, models can be trained with conventional optimizers like Adam \cite{kingma2014adam}. Following previous work, we also consider the curvature $K$ as a learnable parameter with a warmup of $10$ epochs,
and we clamp the max norm of vectors to $40$ before any logarithmic or exponential map \cite{skopek2019mixed}. Appendix \ref{model_arch_and_hyperparams} contains details on model architectures and implementation details. 

\subsection{Structured Density Estimation}
We first consider structured density estimation in a canonical VAE setting \cite{kingma2013auto}, where we seek to learn rich approximate posteriors using normalizing flows and evaluate the marginal log-likelihood of test data. Following work on hyperbolic VAEs, we test the approaches on a branching diffusion process (BDP) and dynamically binarized MNIST \cite{mathieu2019continuous,skopek2019mixed}. 

To estimate the log likelihood we perform importance sampling using 500 samples from the test set \cite{burda2015importance}. Our results are shown in Tables \ref{table:bdp_table} and \ref{table:mnist_table}.
On both datasets we observe our hyperbolic flows provide improvements when using latent spaces of low dimension.
This result matches theoretical expectations---\eg, that trees can be perfectly embedded in $\mathbb{H}^2_K$---and dovetails with previous work on graph embedding \cite{nickel2017poincare}, thus  highlighting the benefit of leveraging hyperbolic space is most prominent in small dimensions. However, as we increase the latent dimension, the Euclidean approaches can compensate for this intrinsic geometric limitation. In the case of BDP we note that the data is indeed a noisy binary tree, which theoretically can be represented in a 2-D hyperbolic space and thus moving to higher dimensional latent space is not beneficial.  
\cut{
and we observe significant gains on both datasets when the latent dimension is small.\cut{In particular, $\mathcal{W}\mathbb{H}C$ performs the best with 2-dimensional latents with a relative improvement of \red{\%XX} over $\mathcal{N}C$ normalizing flows.} We reconcile this result by recalling that even in two dimensions hyperbolic spaces have more room to embed hierarchies. As we increase the latent dimension however we see that Euclidean based approaches outperform our proposed models which are line with similar observation in prior work \cite{nickel2017poincare}.}

\begin{table}[ht]
\begin{small}
\begin{center}
\begin{tabular}{lccccr}
    \toprule
    Model   &  BDP-2 & BDP-4 & BDP-6\\
    \midrule
    $\mathcal{N}$-VAE & $-55.4_{\pm 0.2}$  & $-55.2_{\pm 0.3}$& $-56.1_{\pm 0.2}$   \\
    $\mathbb{H}$-VAE & $-\textbf{54.9}_{\pm 0.3}$& $-55.4_{\pm 0.2}$ &  $-58.0_{\pm 0.2}$\\
    \cut{$\mathcal{P}$-VAE$^*$ & $-55.6_{\pm 0.2}$ & - &-  \\}
    \cut{$\mathbb{U}$-VAE & &  & $-55.8_{\pm 0.4}$  \\}
    $\mathcal{N}C$ & $-55.4_{\pm 0.4}$ & $ \textbf{-54.7}_{\pm 0.1}$ & $\textbf{-55.2}_{\pm 0.3}$  \\
    $\mathcal{T}C$& $\textbf{-54.9}_{\pm 0.1}$& $-55.4_{\pm 0.1}$& $-57.5_{\pm0.2}$\\
    $\mathcal{W}\mathbb{H}C$& $\textbf{-55.1}_{\pm 0.4}$ & $-55.2_{\pm 0.2}$& $-56.9_{\pm 0.4}$\\
    \bottomrule
\end{tabular}
\caption{Test Log Likelihood on Binary Diffusion Process versus latent dimension. All normalizing flows use 2-coupling layers.}
\label{table:bdp_table}
\end{center}
\vskip -0.1in
\end{small}
\end{table}

\begin{table}[ht]
\begin{small}
\begin{center}
\begin{tabular}{lcccc}
    \toprule
    Model   &  \shortstack{MNIST\\2} & \shortstack{MNIST\\4} & \shortstack{MNIST\\6}  \\
    \midrule
    $\mathcal{N}$-VAE &$-139.5_{\pm 1.0}$& $-115.6_{\pm0.2}$ & $-100.0_{\pm0.02}$ \\
    $\mathbb{H}$-VAE & $*$ & $-113.7_{\pm0.9}$& $-99.8_{\pm0.2}$ \\
    \cut{$\mathcal{P}$-VAE$^*$ & $-142.5_{\pm 0.4}$ & $-97.7_{\pm0.2}$&  \\}
    \cut{$\mathbb{U}$-VAE$^*$ & - & $-97.3_{\pm 0.2}$ &  \\}
    $\mathcal{N}C$ &  $-139.2_{\pm 0.4}$ & $-115.2_{\pm0.6}$& $\textbf{-98.7}_{0.3}$ \\
    $\mathcal{T}C$  & $*$& $ \textbf{-112.5}_{\pm0.2}$&$-99.3_{\pm0.2}$  \\
    $\mathcal{W}\mathbb{H}C$ & $\textbf{-136.5}_{\pm 2.1}$ & $\textbf{-112.8}_{\pm0.5}$ &$-99.4_{\pm0.2}$ \\
    \bottomrule
\end{tabular}
\caption{Test Log Likelihood on MNIST averaged over 5 runs verus latent dimension. * indicates numerically unstable settings.}
\label{table:mnist_table}
\end{center}
\vskip -0.1in
\vspace{-13pt}
\end{small}
\end{table}

\subsection{Graph Reconstruction}
We evaluate the utility of our hyperbolic flows by conducting experiments on the task of link prediction using graph neural networks (GNNs) \cite{scarselli2008graph} as an inference model. Given a simple graph $\mathcal{G}=(\V,A, X)$, defined by a set of nodes $\mathcal{V}$, an adjacency matrix $A \in \mathbb{Z}^{|\mathcal{V}| \times |\mathcal{V}|}$ and node feature matrix $X \in \mathbb{R}^{|\mathcal{V}| \times n}$, we learn a VGAE \cite{kipf2016variational} model whose inference network, $q_\phi$, defines a distribution over node embeddings $q_\phi(Z | A, X)$. To score the likelihood of an edge existing between pairs of nodes we use an inner product decoder: $p(A_{u,v}=1|z_u,z_v) = \sigma(z_u^Tz_v)$, with dot products computed in $\mathcal{T}_{\textbf{o}}\mathbb{H}^n_K$ when necessary. Given these components, the inference GNNs are trained to maximize the variational lower bound on a training set of edges. 

We use two different disease datasets taken from \citep{chami2019hyperbolic} and \citep{mathieu2019continuous}\footnote{We uncovered issues with the two remaining datasets in \cite{mathieu2019continuous} and thus omit them (Appendix \ref{dataset_issues}).} for evaluation purposes. Our chosen datasets reflect important real world use cases where the data is known to contain hierarchies. One such measure to determine how tree-like a given graph is known to be Gromov’s $\delta$-hyperbolicity and traditional link prediction datasets such as Cora and Pubmed \cite{yang2016revisiting} were found to lack such a property and are not suitable candidates to evaluate our proposed approach \cite{chami2019hyperbolic}. The first dataset Diseases-\RNum{1} is composed of a network of disorders and disease genes linked by the known disorder–gene associations \cite{goh2007human}. In the second dataset Diseases-\RNum{2}, we build tree networks of a SIR disease spreading model \cite{anderson1992infectious}, where node features determine the susceptibility to the disease. In Table \ref{graph_embeddings_table} we report the AUC and average precision (AP) on the test set.
We observe consistent improvements when using hyperbolic $\mathcal{W}\mathbb{H}C$ flow. Similar to the structured density estimation setting, the performance gains of $\mathcal{W}\mathbb{H}C$ are best observed in low-dimensional latent spaces.

\begin{table}[ht]
\begin{small}
\begin{center}
\begin{tabular}{lcccc}
    \toprule
    Model   & \shortstack{Dis-\RNum{1}\\AUC} & \shortstack{Dis-\RNum{1}\\AP}  & \shortstack{Dis-\RNum{2}\\AUC} & \shortstack{Dis-\RNum{2}\\AP}  \\
    \midrule
    $\mathcal{N}$-VAE & $0.90_{\pm 0.01}$ &
    $0.92_{\pm 0.01}$ &
    $0.92_{\pm 0.01}$ &
    $0.91_{\pm 0.01}$
    
    \\
    $\mathbb{H}$-VAE & $0.91_{\pm 5\textnormal{e-3}}$ &
    $0.92_{\pm 5\textnormal{e-3}}$ &
    $0.92_{\pm 4\textnormal{e-3}}$ &
    $0.91_{\pm 0.01}$ 
    
    \\
    $\mathcal{N}C$ & $0.92_{\pm 0.01}$ &
    $0.93_{\pm 0.01}$ &
     $0.95_{\pm 4\textnormal{e-3}}$ &
    $0.93_{\pm 0.01}$ 
    
    \\
    $\mathcal{T}C$ & $\textbf{0.93}_{\pm 0.01}$ &
    $0.93_{\pm 0.01}$ &
   $\textbf{0.96}_{\pm 0.01}$ &
     $0.95_{\pm 0.01}$ 
    
    \\
    $\mathcal{W}\mathbb{H}C$ & $\textbf{0.93}_{\pm 0.01}$&
    $\textbf{0.94}_{\pm 0.01}$ &
    $\textbf{0.96}_{\pm 0.01}$ &
    $\textbf{0.96}_{\pm 0.01}$
    \\
    \bottomrule
\end{tabular}
\end{center}
\end{small}
\caption{Test AUC and Test AP on Graph Embeddings where Dis-\RNum{1} has latent dimesion 6 and Dis-\RNum{2} has latent dimension 2.}
\label{graph_embeddings_table}
\vskip -0.1in
\end{table}

\subsection{Graph Generation}

Finally, we explore the utility of our hyperbolic flows for generating hierarchical structures. 
As a synthetic testbed, we construct datasets containing uniformly random trees as well as uniformly random lobster graphs \cite{golomb1996polyominoes}, where each graph contains between 20 to 100 nodes. Unlike prior work on graph generation---\textit{i.e.,} \cite{liu2019graph}---our datasets are designed to have explicit hierarchies, thus enabling us to test the utility of hyperbolic generative models.
We then train a generative model to learn the distribution of these graphs. 
We expect the hyperbolic flows to provide a significant benefit for generating valid random trees, as well as learning the distribution of lobster graphs, which are a special subset of trees. 

We follow the two-stage training procedure outlined in Graph Normalizing Flows \cite{liu2019graph} in that we first train an autoencoder to give node-level latents on which we train an normalizing flow for density estimation. Empirically, we find that using GRevNets \cite{liu2019graph} and defining edge probabilities using a distance-based decoder consistently leads to better generation performance. Thus, we define edge probabilities as $p(A_{u,v}=1|z_u,z_v) = \sigma((-d_{\mathcal{G}}(u,v) - b)/\tau)$ where $b$ and $\tau$ are learned edge specific bias and temperature parameters. At inference time, we first sample the number of nodes to generate from the empirical distribution of the dataset. We then independently sample node latents from our prior, beginning with a fully connected graph, and then push these samples through our learned flow to give refined edge probabilities. 

\begin{figure*}
    \centering
    \includegraphics[width=\textwidth]{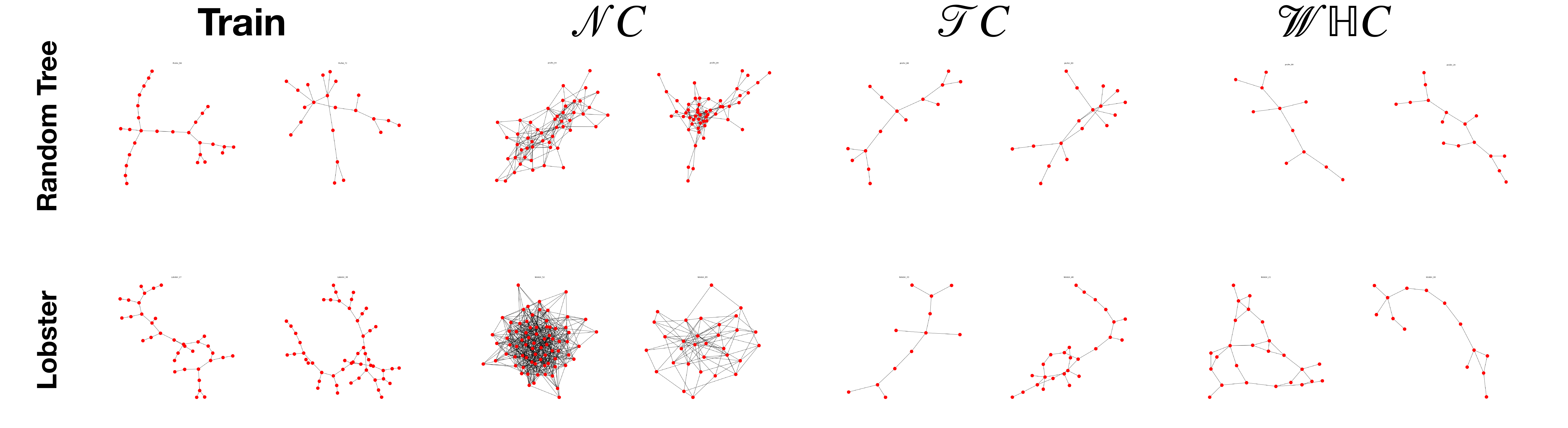}
    \vspace{-5mm}
    \caption{Selected qualitative results on graph generation for lobster and random tree graph.}
    \label{fig:graph_generation_pic}
\end{figure*}

To evaluate the various approaches, we construct $100$ training graphs for each dataset to train our model.
Figure \ref{fig:graph_generation_pic} shows representative samples generated by the various approaches.
We see that hyperbolic normalizing flows learn to generate tree-like graphs and also match the specific properties of the lobster graph distribution, whereas the Euclidean flow model tends to generate densely connected graphs with many cycles (or else disconnected graphs). 
To quantify these intuitions, Table \ref{tab:randtrees} contains statistics on how often the different models generate valid trees (denoted by ``accuracy''), as well as the average number of triangles and the average global clustering coefficients for the generated graphs. 
Since the target data is random trees, a perfect model would achieve 100\% accuracy, with no triangles, and a global clustering of 0 for all graphs. As a representative Euclidean baseline we employ Graph Normalizing Flows (GNFs) which is denoted as $\mathcal{N}C$ in Table \ref{tab:randtrees} and Figure \ref{fig:lobster_graph_gen}.
We see that the hyperbolic models generate valid trees more often, and they generate graphs with fewer triangles and lower clustering on average.
Finally, to evaluate how well the models match the specific properties of the lobster graphs, we follow \citet{liao2019efficient} and report the MMD distance between the generated graphs and a test set for various graph statistics (Figure \ref{fig:lobster_graph_gen}).
Again, we see that the hyperbolic approaches significantly outperform the Euclidean normalizing flow. 

\begin{table}[]
\label{graph_gen_table}
\begin{center}
\begin{tabular}{llll}
    \toprule
    Model   & Accuracy & Avg. Clust. & Avg. GC.\\
    \midrule
    $\mathcal{N}C$ & $56.6_{\pm 5.5}$ & $40.9_{\pm 42.7}$ & $0.34_{\pm0.10}$\\
    $\mathcal{T}C$ & $32.1_{\pm 1.9}$ & $98.3_{\pm 89.5}$ & $0.25_{\pm 0.12}$\\
    $\mathcal{W}\mathbb{H}C$ & $\textbf{62.1}_{\pm 10.9}$ & $\textbf{21.1}_{\pm 13.4}$ & $\textbf{0.13}_{\pm0.07}$\\
    \bottomrule
\end{tabular}
\end{center}
\caption{Generation statistics on random trees over $5$ runs.}
\label{tab:randtrees}
\vspace{-15pt}
\end{table}

\begin{figure}
    \centering
    \includegraphics[width=0.85\linewidth]{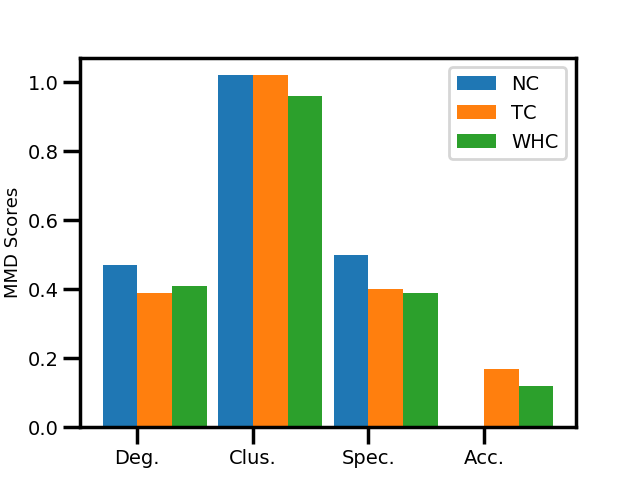}
    \vspace{-15pt}
    \caption{MMD scores for graph generation on Lobster graphs. Note, that $\mathcal{N}C$ achieves $0\%$ accuracy.}
    \label{fig:lobster_graph_gen}
    \vspace{-15pt}
\end{figure}

%% file: relatedwork.tex
\section{Related Work}
\xhdr{Hyperbolic Geometry in Machine Learning:}
The intersection of hyperbolic geometry and machine learning has recently risen to prominence \cite{dhingra2018embedding,tay2018hyperbolic,law2019lorentzian,khrulkov2019hyperbolic,ovinnikov2019poincar}. Early prior work proposed to embed data into the Poincar\'e ball model \cite{nickel2017poincare,chamberlain2017neural}.
The equivalent Lorentz model was later shown to have better numerical stability properties \cite{nickel2018learning}, and recent work has leveraged even more stable tiling approaches \cite{yu2019numerically}. 
In addition, there exists a burgeoning literature of hyperbolic counterparts to conventional deep learning modules on Euclidean spaces (\textit{e.g.}, matrix multiplication), enabling the construction of hyperbolic neural networks (HNNs) \cite{gulcehre2018hyperbolic,ganea2018hyperbolic} with further extensions to graph data using hyperbolic GNN architectures \cite{liu2019graph,chami2019hyperbolic}.
Latent variable models on hyperbolic space have also been investigated in the context of VAEs, using  generalizations of the normal distribution \cite{nagano2019wrapped,mathieu2019continuous}.\cut{As an extension, \cite{skopek2019mixed} consider a more general case where the latent space of a VAE is a product of Riemannian manifolds with constant and learnable curvatures.} In contrast, our work learns a flexible approximate posterior using a novel normalizing flow designed to use the geometric structure of hyperbolic spaces.
In addition to work on hyperbolic VAEs, there are also several works that explore other non-Euclidean spaces (e.g., spherical VAEs) \cite{davidson2018hyperspherical,falorsi2019reparameterizing,grattarola2019adversarial}.

\xhdr{Learning Implicit Distributions}
In contrast with exact likelihood methods there is growing interest in learning implicit distributions for generative modelling. Popular approaches include density ratio estimation methods using a parametric classifiers such as GANS \cite{goodfellow2014generative}, and kernel based estimators \cite{shi2017kernel}. In the context of autoencoders learning implicit latent distribution can be seen as an adversarial game minimizing a specific divergence \cite{makhzani2015adversarial} or distance \cite{tolstikhin2017wasserstein}. Instead of adversarial formulations implicit distributions may also be learned directly by estimating the gradients of log density function using the Stein gradient estimator \cite{li2017gradient}. Finally, such gradient estimators can also be used to power variational inference with implicit posteriors enabling the use of posterior families with intractable densities  \cite{shi2018spectral}.

\xhdr{Normalizing Flows:} 
Normalizing flows (NFs)~\cite{rezende2015variational,dinh2016density} are a class of probabilistic models which use invertible transformations to map samples from a simple base distribution to samples from a more complex learned distribution.\cut{The new density can then be efficiently computed via the change of variable formula making normalizing flow powerful tools for variational inference and generative modeling.} While there are many classes of normalizing flows \cite{papamakarios2019normalizing,kobyzev2019normalizing}, our work largely follows normalizing flows designed with partially-ordered dependencies, as found in affine coupling transformations \cite{dinh2016density}. 
Recently, normalizing flows have also been extended to Riemannian manifolds, such as spherical spaces in \citet{gemici2016normalizing}. In parallel to this work, normalizing flows have been extended to toriodal spaces \cite{rezende2020normalizing} and the data manifold \cite{brehmer2020flows}. Finally, relying on affine coupling and GNNs, \citet{liu2019graph} develop graph normalizing flows (GNFs) for generating graphs. However, unlike our approach GNFs do not benefit from the rich geometry of hyperbolic spaces. 

%% file: conclusion.tex
\section{Conclusion}
In this paper, we introduce two novel normalizing flows on hyperbolic spaces. 
We show that our flows are efficient to sample from, easy to invert and require only $\mathcal{O}(n)$ cost to compute the change in volume. We demonstrate the effectiveness of constructing hyperbolic normalizing flows for latent variable modeling of hierarchical data. We empirically observe improvements in structured density estimation, graph reconstruction and also generative modeling of tree-structured data, with large qualitative improvements in generated sample quality compared to Euclidean methods. One important limitation is in the numerical error introduced by clamping operations which prevent the creation of deep flow architectures. We hypothesize that this is an inherent limitation of the Lorentz model, which may be alleviated with newer models of hyperbolic geometry that use integer-based tiling \cite{yu2019numerically}. In addition, while we considered hyperbolic generalizations of the coupling transforms to define our normalizing flows, designing new classes of invertible transformations like autoregressive and residual flows on non-Euclidean spaces is an interesting direction for future work.

%% file: appendix.tex
\appendix
\onecolumn
\section{Background on Riemannian Geometry}
\label{riemannian_geometry_appendix}
An $n$-dimensional manifold is a topological space that is equipped with a family of open sets $U_i$ which cover the space and a family of functions $\psi_i$ that are homeomorphisms between the $U_i$ and open subsets of $\mathbb{R}$.  The pairs $(U_i,\psi_i)$ are called \emph{charts}.  A crucial requirement is that if two open sets $U_i$ and $U_j$ intersect in a region, call it $U_{ij}$, then the composite map $\psi_i\circ \psi_j^{-1}$ restricted to $U_{ij}$ is infinitely differentiable. If $\mathcal{M}$ is an $n-$dimensional manifold then a chart, $\psi: U \to V$, on $\mathcal{M}$ maps an open subset $U$ to an open subset $V \subset \mathbb{R}^n$. Furthermore, the image of the point $p \in U$, denoted $\psi(p) : \mathbb{R}^n$ is termed the local coordinates of $p$ on the chart $\psi$. Examples of manifolds include $\mathbb{R}^n$, the Hypersphere $\mathbb{S}^n$, the Hyperboloid $\mathbb{H}^n$, a torus. In this paper we take an extrinsic view of the geometry, that is to say a manifold can be thought of as being embedded in a higher dimensional Euclidean space, ---i.e. $\mathcal{M}^n \subset \mathbb{R}^{n+1}$, and inherits the coordinate system of the ambient space. This is not how the subject is usually developed but for spaces of constant curvature one gets convenient formulas.

\xhdr{Tangent Spaces} Let $p \in \mathcal{M}$ be a point on an $n-$dimensional smooth manifold and let $\gamma(t) \to \mathcal{M}$ be a differentiable parametric curve with parameter $t \in [-\epsilon, \epsilon]$ passing through the point such that $\gamma(0) = p$. Since $\mathcal{M}$ is a smooth manifold we can trace the curve in local coordinates via a chart $\psi$ and the entire curve is given in local coordinates by $x = \psi \circ \gamma(t)$. The tangent vector to this curve at $p$ is then simply $v = (\psi \circ \gamma)'(0)$. Another interpretation of the tangent vector of $\gamma$ is by interpreting the point $p$ as a position vector and the tangent vector is then interpreted as the velocity vector at that point.\cut{Stated another way the tangent vector is a vector of partial derivatives for every coordinate.} Using this definition the set of all tangent vectors at $p$ is denoted as $\mathcal{T}_p{\mathcal{M}}$, and is called the tangent space at $p$. 

\xhdr{Riemannian Manifold}
A Riemannian metric tensor $g$ on a smooth manifold $\mathcal{M}$ is defined as a family of inner products such that at each point $p \in \mathcal{M}$ the inner product takes vectors from the tangent space at $p$, $g_p= \langle \cdot, \cdot \rangle_p: \mathcal{T}_p\mathcal{M} \times \mathcal{T}_p\mathcal{M} \to \mathbb{R}$. This means $g$ is defined for every point on $\mathcal{M}$ and varies smoothly. Locally, $g$ can be defined using the basis vectors of the tangent space $g_{ij}(p) = g(\frac{\partial}{\partial p_i}, \frac{\partial}{\partial p_j})$. In matrix form the Riemannian metric, $G(p)$, can be expressed as, $\forall u,v \in  \mathcal{T}_p\mathcal{M} \times \mathcal{T}_p\mathcal{M}, \langle u, v \rangle_p = g(p)(u,v) = u^T G(p) v$. A smooth manifold manifold $\mathcal{M}$ which is equipped with a Riemannian metric at every point $p \in \mathcal{M}$ is called a Riemannian manifold. Thus every Riemannian manifold is specified as the tuple $(\mathcal{M},g)$ which define the smooth manifold and its associated Riemannian metric tensor.

Armed with a Riemannian manifold we can now recover some conventional geometric insights such as the length of a parametric curve $\gamma$, the distance between two points on the manifold, local notion of angle, surface area and volume. We define the length of a curve, $L[\gamma] = \int_a^b g_{\gamma (t)} || \gamma '(t)|| dt$. This definition is very similar to the length of a curve on Euclidean spaces if we just observe that the Riemannian metric is $I_n$. Now turning to the distance between points $p$ and $q$ we can reason that it must be the smallest or distance minimizing parametric curve between the points which in the literature are known as \textit{geodesics}\footnote{Actually a geodesic is usually defined as a curve such that the tangent vector is parallel transported along it. It is then a theorem that it gives the shortest path.}. Stated another way: $d(p,q) = \inf \big\{ L[\gamma] \ | \gamma:[a,b] \to \mathcal{M}\big\}$ with ,  $\gamma(a)=p$ and $\gamma(b)=q$. A norm is induced on every tangent space by $g_p$ and is defined as $\mathcal{T}_p\mathcal{M}: || \cdot ||_p :\sqrt{\langle \cdot, \cdot \rangle_p}$. Finally, we can also define an infitisimal volume element on each tangent space and as a result measure $d\mathcal{M}(p) = \sqrt{|G(p)|} dp$, with $dp$ being the Lebesgue measure. 

\section{Background Normalizing Flows}
\label{normalizing_flow_appendix}
Given a parametrized density on $\mathbb{R}^n$ a \textit{normalizing flow} defines a sequence of invertible transformations to a more complex density over the same space via the change of variable formula for probability distributions \cite{rezende2015variational}. Starting from a sample from a base distribution, $z_0 \sim p(z)$, a mapping $f: \mathbb{R}^d \to \mathbb{R}^d$, with parameters $\theta$ that is both invertible and smooth, the log density of $z' = f(z_0)$ is defined as $\log p_{\theta}(z') = \log p(z_0) - \log \det \Big \lvert \frac{\partial f}{\partial z} \Big \rvert$.
\cut{
\begin{align}
    \log p_{\theta}(z') = \log p(z_0) - \log \det \Big \lvert \frac{\partial f}{\partial z} \Big \rvert.
    \label{flow_eqn_1}

\end{align}
}
Where, $p_{\theta}(z')$ is the probability of the transformed sample and $\partial f / \partial x$ is the Jacobian of $f$. To construct arbitrarily complex densities a chain of functions of the same form as $f$ can be defined and through successive application of change of density for each invertible transformation in the flow. Thus the final sample from a flow is then given by $z_j = f_j \circ f_{j-1} ... \circ f_1(z_0)$ and it's corresponding density can be determined simply by $\ln p_{\theta}(z_j) = \ln p(z_0) - \sum_{i=1}^j\ln \det \Big \lvert \frac{\partial f_i}{\partial z_{i-1}} \Big \rvert$. 
Of practical importance when designing normalizing flows is the cost associated with computed the log determinant of the Jacobian which is computationally expensive and can range anywhere from $O(n!)-O(n^3)$ for an arbitrary matrix and a chosen algorithm. However, through an appropriate choice of $f$ this computation cost can be brought down significantly. While there are many different choices for the transformation function, $f$, in this work we consider only RealNVP based flows as presented in \cite{dinh2016density} and \cite{rezende2015variational} due to their simplicity and expressive power in capturing complex data distributions.

\subsection{Variational Inference with Normalizing Flows}
One obvious use case for Normalizing Flows is in learning a more expressive often multi-modal posterior distribution needed in Variational Inference. Recall that a variational approximation is a lower bound to the data log-likelihood. Take for example amortized variational inference in a VAE like setting whereby the posterior $q_{\theta}$ is parameterized and is amenable to gradient based optimization. The overall objective with both encoder and decoder networks:
\begin{align}
    \log p(x) & = \log \int p(x|z)p(z) dz \\
              & \geq \mathbb{E}_{q_{\theta}(z|x)}[\log \frac{p(x,z)}{q_{\theta}(z|x)}] \ \ \ \  (\textnormal{Jensen's Inequality}) \\
              & = \mathbb{E}_{q_{\theta}(z|x)} [\log p(x|z)] + \mathbb{E}_{q_{\theta}(z|x)}\Big[\log \frac{p(z)}{q_{\theta}(z|x)}\Big] \\
              & = \mathbb{E}_{q_{\theta}(z|x)} [\log p(x|z)] - D_{KL}(q_{\theta}(z|x) || p(z)) \label{nn_vae_perspective}
\end{align}
The tightness of the Evidence Lower Bound (ELBO) also known as the negative free energy of the system, $-\mathcal{F}(x)$, is determined by the quality of the posterior approximation to the true posterior. Thus, one way to enrich the posterior approximation is by letting $q_{\theta}$ be a normalizing flow itself and the resultant latent code be the output of the transformation. If we denote $q_0(z_0)$ the probability of the latent code $z_0$ under the base distribution and $z_k$ as the latent code after $K$ flow layers we may rewrite the Free Energy as follows:

\begin{align}
    \mathcal{F}(x) &= \mathbb{E}_{q_{0}(z_0)}[\log q_k(z_j) - \log p(x,z_j)] \\
    &= \mathbb{E}_{q_{0}(z_0)}\Big[\log q_0(z_0) - \sum_{i=1}^j\ln \det \Big \lvert \frac{\partial f_i}{\partial z_{i-1}} \Big \rvert - \log p(x,z_i)\Big]\\
    &= D_{KL}(q_0(z_0)|| p(z_j)) - \mathbb{E}_{q_{0}(z_0)}\Big[\sum_{i=1}^j\ln \det \Big \lvert \frac{\partial f_i}{\partial z_{i-1}} \Big \rvert - \log p(x|z_i)\Big]
\end{align}
For convenience we may take $q_0 = \mathcal{N}(\mu, \sigma^2)$ which is a reparametrized gaussian density and $p(z) = \mathcal{N}(0, I)$ a standard normal. 

\subsection{Euclidean RealNVP}
\label{Euclidean_RealNVP_appendix}
Computing the Jacobian of functions with high-dimensional domain and codomain and computing the determinants of large matrices are in general computationally very expensive. Further complications can arise with the restriction to bijective functions make for difficult modelling of arbitrary distributions. A simple way to significantly reduce the computational burden is to design transformations such that the Jacobian matrix is triangular resulting in a determinant which is simply the product of the diagonal elements. In \cite{dinh2016density}, real valued non-volume preserving (RealNVP) transformations are introduced as simple bijections that can be stacked but yet retain the property of having the composition of transformations having a triangular determinant. To achieve this each bijection updates a part of the input vector using a function that is simple to invert,
but which depends on the remainder of the input vector in a complex way. Such transformations are denoted  as affine coupling layers. Formally, given a $D$ dimensional input $x$ and $d < D$, the output $y$ of an affine coupling layer follows the equations:
\begin{align}
    y_{1:d} & = x_{1:d} \\
    y_{d+1:D} & = x_{d+1:D} \odot \textnormal{exp}(s(x_{1:d})) + t(x_{1:d}).
\end{align}
Where, $s$ and $t$ are parameterized scale and translation functions. As the second part of the input depends on the first, it is easy to see that the Jacobian given by this transformation is lower triangular. Similarly, the inverse of this transformation is given by:
\begin{align}
     x_{1:d} & = y_{1:d} \\
    x_{d+1:D} & = (y_{d+1:D} - t(y_{1:d})\odot \textnormal{exp}(-s(y_{1:d})).
\end{align}
Note that the form of the inverse does not depend on calculating the inverses of either $s$ or $t$ allowing them to be complex functions themselves. Further note that with this simple bijection part of the input vector is never touched which can limit the expressiveness of the model. A simple remedy to this is to simply reverse the elements that undergo scale and translation transformations prior to the next coupling layer. Such an alternating pattern ensures that each dimension of the input vector depends in a complex way given a stack of couplings allowing for more expressive models. Finally, the Jacobian of this transformation is a lower triangular matrix,
\begin{equation} 
    \frac{\partial y}{\partial x} = \begin{bmatrix}
                                    \mathbb{I}_{d} & 0 \\
                                
                                   \frac{\partial y_{d+1:D}}{x^T_{1:d}} & \textnormal{diag}(\textnormal{exp}s(x_{1:d}))
                                    \end{bmatrix}.
\end{equation}

\section{Change of Variable for Tangent Coupling}
We now derive the change in volume formula associated with one $\mathcal{T}C$ layer. Without loss of generality we first define a binary mask which we use to partition the elements of a vector at $\mathcal{T}_{\textbf{o}}\mathbb{H}^n_K$ into two sets. Thus $b$ is defined as 
$$ 
b_j = 
\begin{cases}
1 &\textrm{if $j\leq d$}\\
0 &\textrm{otherwise},
\end{cases}
$$
Note that all $\mathcal{T}C$ layer operations exclude the first dimension which is always copied over by setting $b_0 = 1$ and ensures that the resulting sample always remains on $\mathcal{T}_{\textbf{o}}{\mathbb{H}^n_K}$. Utilizing $b$ we may rewrite Equation \ref{eq:tangent_coupling} as,

\label{tangent_coupling_proof_appendix}
\begin{equation}
    \textbf{y} =  \textnormal{exp}^K_{\textbf{o}}\big(b \odot \tilde{x} + (1 - b)\odot(\tilde{x} \odot \sigma(s(b \odot \tilde{x})) + t(b \odot \tilde{x}))\big),
\end{equation}
where $\tilde{x} = \log^K_{\textbf{o}}(x)$ is a point on the tangent space at $\textbf{o}$. Similar to the Euclidean RealNVP, we wish to calculate the jacobian determinant of this overall transformation. We do so by first observing that the overall transformation is a valid composition of functions: $y := \textnormal{exp}^K_{\textbf{o}} \circ f\circ \log^K_{\textbf{o}}(\textbf{x})$, where $z = f(\tilde{x})$ is the flow in tangent space. Utilizing the chain rule and the identity that the determinant of a product is the product of the determinants of its constituents we may decompose the jacobian determinant as,

\begin{equation}
    \textnormal{det}\Big (\frac{\partial \textbf{y}}{\partial \textbf{x}}\Big) =  \textnormal{det}\Big (\frac{\partial \textnormal{exp}^K_{\textbf{o}}(z)}{\partial z}\Big) \cdot \textnormal{det}\Big (\frac{\partial f(\tilde{x})}{\partial \tilde{x}}\Big) \cdot  \textnormal{det}\Big (\frac{\partial \log^K_{\textbf{o}}(\textbf{x})}{\partial \textbf{x}}\Big)
    \label{jac_det_tangent_flow_1}.
\end{equation}
Tackling each term on RHS of Eq. \ref{jac_det_tangent_flow_1} individually, $ \textnormal{det}\Big (\frac{\partial \textnormal{exp}^K_{\textbf{o}}(z)}{\partial z}\Big) = \Big(\frac{R\sinh(\frac{||z||_{\mathcal{L}}}{R})}{||z||_{\mathcal{L}}}\Big)^{n-1}$ as derived in \cite{nagano2019wrapped}. As the logarithmic map is the inverse of the exponential map the jacobian determinant is also the inverse ---i.e.  $\textnormal{det}\Big (\frac{\partial \log^K_{\textbf{o}}(\textbf{x})}{\partial \textbf{x}}\Big) = \Big(\frac{\sinh(||\log^K_{\textbf{o}}(\textbf{x})||_{\mathcal{L}})}{||\log^K_{\textbf{o}}(\textbf{x})||_{\mathcal{L}}}\Big)^{1-n}$. For the middle term in Eq. \ref{jac_det_tangent_flow_1} we proceed by selecting the standard basis $\{e_1, e_2, ... e_n\}$ which is an orthonormal basis  with respect to the Lorentz inner product. The directional derivative with respect to a basis element $e_j$ is computed as follows:

\begin{align*}
    \textnormal{d}f(\tilde{x}) &= \frac{\partial}{\partial \epsilon} \Big |_{\epsilon=0} f(\tilde{x} + \epsilon e_j)\\
    & = \frac{\partial}{\partial \epsilon} \Big |_{\epsilon=0} \{b \odot (\tilde{x} + \epsilon e_j) + (1 - b)\odot((\tilde{x} + \epsilon e_j) \odot \sigma(s(b \odot (\tilde{x} + \epsilon e_j))) + t(b \odot (\tilde{x} + \epsilon e_j)))\} \\
    &= b \odot  e_j + \frac{\partial}{\partial \epsilon} \Big |_{\epsilon=0}\{ (1 - b)\odot((\tilde{x} + \epsilon e_j) \odot \sigma(s(b \odot (\tilde{x} + \epsilon e_j))) + t(b \odot (\tilde{x} + \epsilon e_j)))\} \\
\end{align*}
As $b \in [0,1]^n$ is a binary mask, it is easy to see that if $b_j =1$ then only the first term on the RHS remains and the directional derivative with respect to $e_j$ is simply the basis vector itself. Conversely, if $b_j =0$ then the first term goes to zero and we are left with the second term,

\begin{align*}
    \textnormal{d}f(\tilde{x}) &= \frac{\partial}{\partial \epsilon} \Big |_{\epsilon=0}\{ (1 - b)\odot((\tilde{x} + \epsilon e_j) \odot \sigma(s(b \odot (\tilde{x} + \epsilon e_j))) + t(b \odot (\tilde{x} + \epsilon e_j)))\} \\
    &=  \frac{\partial}{\partial \epsilon} \Big |_{\epsilon=0}\{ (1 - b)\odot((\tilde{x} + \epsilon e_j) \odot \sigma(s(b \odot \tilde{x})) + t(b \odot \tilde{x}))\} \\
    &= e_j \odot \sigma(s(b \odot \tilde{x})).
\end{align*}
Where in the second line we've used the fact $b \odot \epsilon e_j = 0$. All together, the directional derivatives computed using our chosen basis elements are,

\begin{equation*}
    \textnormal{d}f(\tilde{x}) = (e_1, e_2, \dots e_d, e_{d+1}\odot \sigma(s(b \odot \tilde{x})), \dots e_{D}\odot \sigma(s(b \odot \tilde{x}))).
\end{equation*}
The volume factor given by this linear map is $\textnormal{det}( \textnormal{d}f(\tilde{x})) = \sqrt{G^TG}$, where $G$ is the matrix of all directional derivatives. As the basis elements are orthogonal all non-diagonal entries of $G^TG$ go to zero and the determinant is the product of the Lorentz norms of each component. As $||e_j||_{\mathcal{L}} =1$ and $||e_{j}\odot \sigma(s(b \odot \tilde{x}))||_{\mathcal{L}} = ||e_{j}\odot \sigma(s(b \odot \tilde{x}))||_2$ for $\mathcal{T}_{\mathbf{o}}H^n_K$ the overall determinant is then $\textnormal{d}f(\tilde{x}) = \textnormal{diag} \ \sigma(s(b \odot \tilde{x}))$. Finally, the full log jacobian determinant of a $\mathcal{T}C$ layer is given by,

\begin{equation}
    \log \textnormal{det}\Big (\frac{\partial \textbf{y}}{\partial \textbf{x}}\Big) = \Big(\frac{R\sinh(\frac{||z||_{\mathcal{L}}}{R})}{||z||_{\mathcal{L}}}\Big)^{n-1} + \sum_{i=d+1}^n\sigma(s(\tilde{x}_1))_i 
      + \Big(\frac{R\sinh(\frac{||\log^K_{\textbf{o}}(\textbf{x})||_{\mathcal{L}}}{R})}{||\log^K_{\textbf{o}}(\textbf{x})||_{\mathcal{L}}}\Big)^{1-n}
\end{equation}

Thus the overall computational cost is only slightly larger than the regular Euclidean RealNVP, $\mathcal{O}(n)$.

\section{Change of Variable for Wrapped Hyperbolic Coupling}
\label{wrapped_coupling_proof_appendix}
We consider the following function $f : \mathbb{H}^n_K \rightarrow \mathbb{H}^n_K$, which we use to define a normalizing flow in $n$-dimensional hyperbolic space (represented via the Lorentz model): 
\begin{equation}\label{eq:mainflow}
    f(\mb{x}) = \textnormal{exp}^K_{\textbf{o}}\left( b \odot \tilde{
    x} + (1-b) \odot \log^K_{\textbf{o}}\Big( \textnormal{exp}^K_{t(b \odot \tilde{x})}\big(\textnormal{PT}^K_{\textbf{o}\to t(b \odot \tilde{x}) }((1-b) \odot \tilde{x} \odot \sigma(s(b \odot \tilde{x})))\big)\Big) \right),
\end{equation}

where $\tilde{x} = \log_\mb{o}(\mb{x}) \in \mc{T}_\mb{o}\mbb{H}^n_K$ is the projection of $\mb{x} \in \mbb{H}^n_K$ to the tangent space at the origin, i.e, $\mc{T}_\mb{o}\mbb{H}^n_K$. As in $\mathcal{T}C$ we again utilize a binary mask $b$ so that
$$ 
b_j = 
\begin{cases}
1 &\textrm{if $j\leq d$}\\
0 &\textrm{otherwise},
\end{cases}
$$
where $0 < d < n$.
In Equation \eqref{eq:mainflow} the function $s : \mc{T}_o\mbb{H}^d_K \rightarrow \mc{T}_o\mbb{H}^{n-d}_K$ is an arbitrary function on the tangent space at the origin and $\sigma$ denotes the logistic function.
The function $t : \mc{T}_o\mbb{H}^d_K \rightarrow \mbb{H}^*_K \subset {H}^n_K$ is a map from the tangent space at the origin to a subset of hyperbolic space defined by the set of points satisfying the condition that $\mb{v}_i = 0, \forall i=2...d, \mb{v}_i \in \mbb{H}^n_K$ (under their representation in the Lorentz model).

Our goal is to derive the Jacobian determinant of $f$, i.e., 
\begin{equation}
    \left|\textrm{det}\left(\frac{\partial f(\mb{x})}{\partial \mb{x}}\right)\right|,
\end{equation}
To do so, we will use the following facts without proof or justification:
\begin{itemize}
    \item 
    \textbf{Fact 1: }The chain rule for determinants, i.e., the fact that
    \begin{equation}
        \left|\textrm{det}\left(\frac{\partial f(\mb{x})}{\partial \mb{x}}\right)\right| =  \left|\textrm{det}\left(\frac{\partial f(\mb{x})}{\partial \mb{v}}\right)\right| \left|\textrm{det}\left(\frac{\partial \mb{v}}{\partial \mb{x}}\right)\right|,
    \end{equation}
    where $\mb{v}$ is introduced via a valid change of variables. 
    \item 
    \textbf{Fact 2: } The Jacobian determinant for the exponential map $\exp^K_\mb{u}(z) = \mc{T}_{u}\mbb{H}^n_K \rightarrow \mbb{H}^n_K$ is given by
    \begin{equation}
        \left|\textrm{det}(\exp^K_\mb{u}(z))\right| =
        \left(\frac{R\sinh(\frac{||z||_{\mathcal{L}}}{R})}{||z||_{\mathcal{L}}}\right)^{n-1}
    \end{equation}
        \item 
    \textbf{Fact 3: } The Jacobian determinant for the logarithmic map $\log^K_\mb{u}(\mb{v}) = \mbb{H}^n_K \rightarrow \mc{T}_u\mbb{H}^n_K$ is given by
    \begin{equation}
        \left|\textrm{det}(\log^K_\mb{u}(\mb{v})\right| =
        \left(\frac{R\sinh(\frac{||\log^K_{\textbf{o}}(\textbf{v})||_{\mathcal{L}}}{R})}{||\log^K_{\textbf{o}}(\textbf{v})||_{\mathcal{L}}}\right)^{1-n}
    \end{equation}
    \item
     \textbf{Fact 4: } The Jacobian determinant for parallel transport $\textnormal{PT}^K_{\mb{u} \rightarrow \mb{t}}(v) = \mc{T}_\mb{u}\mbb{H}^n_K \rightarrow \mc{T}_t\mbb{H}^n_K$ is given by
    \begin{equation}
        \left|\textrm{det}\left(\textnormal{PT}^K_{\mb{u} \rightarrow \mb{t}}(v)\right)\right| = 1.
    \end{equation}
\end{itemize}
Fact 2 and Fact 4 are proven in \citet{nagano2019wrapped} ``A Wrapped Normal Distribution on Hyperbolic Space for Gradient-Based Learning'' for $K=-1$ and rederived for general $K$ in \citet{skopek2019mixed}. Fact 3 follows from the fact that the determinant of the inverse of a function is the inverse of that function's determinant. 
We will use similar arguments to obtain our determinant as were used in \citet{nagano2019wrapped}, and we refer the reader to Appendix A.3 in their work for background. 

Our main claim is as follows
\begin{prop}
The Jacobian determinant of the function $\tilde{f}^{\mathcal{W}\mathbb{H}C}$ in \eqref{wrapped_hyperboloid_coupling_eqn} is:
\begin{multline}
  \left|\textnormal{det}\left(\frac{\partial\mb{y}}{\partial \mb{x}}\right)\right| = \prod_{i=d+1}^n\sigma(s(\tilde{x}_1))_i \times \Big(\frac{R \sinh(\frac{||q||_{\mathcal{L}}}{R})}{||q||_{\mathcal{L}}}\Big)^{l}
  \times \Big(\frac{R\sinh(\frac{||\log_{\textbf{o}}^K(\hat{\textbf{q}})||_{\mathcal{L}}}{R})}{||\log_{\textbf{o}}^K(q)||_{\mathcal{L}}}\Big)^{-l} \ \times  \Big(\frac{R\sinh(\frac{||\tilde{z}||_{\mathcal{L}}}{R})}{||\tilde{z}||_{\mathcal{L}}}\Big)^{n-1}
 \  \\ \times \Big(\frac{R\sinh(\frac{||\log_{\textbf{o}}^K(\textbf{x})||_{\mathcal{L}}}{R})}{||\log_{\textbf{o}}^K(\textbf{x})||_{\mathcal{L}}}\Big)^{1-n},
\end{multline}
where 
$$
z =  b \odot \tilde{
    x} + \log^K_{\textbf{o}}\Big( \textnormal{exp}^K_{t(b \odot \tilde{x})}\big(\textnormal{PT}^K_{\textbf{o}\to t(b \odot \tilde{x}) }((1-b) \odot \tilde{x} \odot \sigma(s(b \odot \tilde{x})))\big)\Big)
$$
the argument to the parallel transport $q$ is,
$$
 q = \textnormal{PT}^K_{\textbf{o}\to t(b \odot \tilde{x}) }((1-b) \odot \tilde{x} \odot \sigma(s(b \odot \tilde{x}))). 
$$
and
$$
\hat{\textbf{q}} = \textnormal{exp}_{\textbf{t}}^K(q)
$$
\end{prop}
\begin{proof}
We first note that 
\begin{equation}
\left|\textnormal{det}\left(\frac{\partial f(\mb{x})}{\partial \mb{x}}\right)\right| = \left|\textrm{det}\left(\frac{\partial f(\mb{x})}{\partial z}\right)\right| \times \left|\textrm{det}\left(\frac{\partial z}{\partial \tilde{x}}\right)\right| \times \left|\textrm{det}\left(\frac{\partial \tilde{x}}{\partial \mb{x}}\right)\right| 
\end{equation}
by the chain rule (recalling that $\tilde{x} = \log_\mb{o}(\mb{x})$).
Now, we have that 
\begin{equation}
\left|\textrm{det}\left(\frac{\partial f(\mb{x})}{\partial z}\right)\right| =
\Big(\frac{R\sinh(\frac{||z||_{\mathcal{L}}}{R})}{||z||_{\mathcal{L}}}\Big)^{n-1}
\end{equation}
by Fact 2.
And, 
\begin{equation}
\left|\textrm{det}\left(\frac{\partial \tilde{x}}{\partial \mb{x}}\right)\right| =
\Big(\frac{R\sinh(\frac{||\log^K_{\textbf{o}}(\textbf{x})||_{\mathcal{L}}}{R})}{||\log^K_{\textbf{o}}(\textbf{x})||_{\mathcal{L}}}\Big)^{1-n}
\end{equation}
by Fact 3.
Thus, we are left with the term
$$
\left|\textrm{det}\left(\frac{\partial z}{\partial \tilde{x}}\right)\right|.
$$
To evaluate this term, we rely on the following Lemma:

\begin{lemma}
Let $h : \mc{T}_\mb{o}\mbb{H}^n_K \rightarrow \mc{T}_\mb{o}\mbb{H}^n_K$ be a function from the tangent space at the origin to the tangent space at the origin defined as:
\begin{equation}\label{eq:g}
h(\tilde{x}) = z = b \odot \tilde{
    x} + \log^K_{\textbf{o}}\Big( \textnormal{exp}^K_{t(b \odot \tilde{x})}\big(\textnormal{PT}^K_{\textbf{o}\to t(b \odot \tilde{x}) }((1-b) \odot \tilde{x} \odot \sigma(s(b \odot \tilde{x})))\big)\Big).
\end{equation}
Now, define a function $h^* : \mc{T}_\mb{o}\mbb{H}^{n-d}_K \rightarrow \mc{T}_\mb{o}\mbb{H}^{n-d}_K$ which acts on the subspace of $\mc{T}_\mb{o}\mbb{H}^{n-d}_K$ corresponding to the standard basis elements $e_{d+1}, ..., e_n$ as
\begin{equation}\label{eq:gstar}
h^*(\tilde{x}_{d+1:n}) =   \log^K_{\mb{o}_{d+1:n}}\Big( \textnormal{exp}^K_{t_{d+1:n}}\big(\textnormal{PT}^K_{\mb{o}_{d+1:n} \to t_{d+1:n}}( \tilde{x}_{d+1:n} \odot \sigma(s))\big)\Big),
\end{equation}
where $\tilde{x}_{d+1:n}$ denotes the portion of the vector $\tilde{x}$ corresponding to the standard basis elements $e_{d+1}, ..., e_n$ and $s$ and $t$ are constants (which depend on $\tilde{x}_{2:d}$).
In \eqref{eq:gstar}, we use $\mb{o}_{d+1:n} \in \mbb{H}^{n-d}_K$ to denote the vector corresponding to only the dimensions $d+1, ..., n$ and similarily for $t_{d+1:n}$.
Then we have that
\begin{equation}
    \left|\textnormal{det}\left(\frac{\partial z}{\partial \tilde{x}}\right)\right| =    \left|\textnormal{det}\left(\frac{\partial h^*(\tilde{x}_{d+1:n})}{\partial \tilde{x}_{d+1:n})}\right)\right|.
\end{equation}
\end{lemma}
\begin{proof}
First note that by design we have that 
\begin{equation}
    [0,0..,0] \oplus h^*(\tilde{x}_{d+1:n}) = \log^K_{\textbf{o}}\Big( \textnormal{exp}^K_{t(b \odot \tilde{x})}\big(\textnormal{PT}^K_{\textbf{o}\to t(b \odot \tilde{x}) }((1-b) \odot \tilde{x} \odot \sigma(s(b \odot \tilde{x})))\big)\Big),
\end{equation}
i.e., the output of $h^*$ is equal to right hand side of Equation \eqref{eq:g} after prepending/concatenating 0s to the output of $h^*$. 

Now, we can evaluate 
$$
\left|\textrm{det}\left(\frac{\partial z}{\partial \tilde{x}}\right)\right|
$$
by examining the directional derivative with respect to a set of basis elements of $\mc{T}_\mb{o}\mbb{H}^{n}_K$.
Now, given that this is the tangent space at the origin, we know that the standard (i.e., Euclidean) basis elements $e_2, ..., e_n$ form a valid basis for this subspace, since they are orthogonal under the Lorentz normal and orthogonal to the origin itself. 
Now, we can note first that
\begin{equation}
    D_{e_i}h(\tilde{x}) = e_i, \forall i=2...d.
\end{equation}
In other words, the directional derivative for the first $d$ basis elements is the simply the basis elements themselves. This can be verified by taking the definition of the directional derivative:
\begin{equation}
     D_{e_i}h(\tilde{x})  = \frac{\partial}{\partial \epsilon}\Big|_{\epsilon = 0}h(\tilde{x} + \epsilon e_i)
\end{equation}
and noting that the 
$$
\log^K_{\textbf{o}}\Big( \textnormal{exp}^K_{t(b \odot \tilde{x})}\big(\textnormal{PT}^K_{\textbf{o}\to t(b \odot \tilde{x}) }((1-b) \odot \tilde{x} \odot \sigma(s(b \odot \tilde{x})))\big)\Big)
$$
term must equal zero since $(1-b)\odot e_i = 0, \forall i = 2, ..., d$ by design. 
Now, for the basis elements $e_i$ with $i>d$ we have that
\begin{equation}
     D_{e_i}h(\tilde{x}) \perp e_j, \forall i=d+1, ..., n, j=2...,d.
\end{equation}
This holds because 
\begin{align}
D_{e_i}h(\tilde{x})  &= \frac{\partial}{\partial \epsilon}\Big|_{\epsilon = 0}h(\tilde{x} + \epsilon e_i) \\
 &= \frac{\partial}{\partial \epsilon}\Big|_{\epsilon = 0}    \log^K_{\textbf{o}}\Big( \textnormal{exp}^K_{t(b \odot \tilde{x})}\big(\textnormal{PT}^K_{\textbf{o}\to t(b \odot \tilde{x}) }((1-b) \odot \tilde{x} \odot \sigma(s(b \odot \tilde{x})))\big)\Big) 
\end{align}
since $b \odot e_i = 0, \forall i = d+1, ..., n$ by design and because
\begin{equation}
   \log^K_{\textbf{o}}\Big( \textnormal{exp}^K_{t(b \odot \tilde{x})}\big(\textnormal{PT}^K_{\textbf{o}\to t(b \odot \tilde{x}) }((1-b) \odot \tilde{x} \odot \sigma(s(b \odot \tilde{x})))\big)\Big) \perp e_j, \forall \tilde{x} \in \mc{T}_\mb{o}\mbb{H}^n_K, \forall  j=2...,d.
\end{equation}
due to the $(1-b)$ term inside the parallel transport and by our design of the function $t$. 
Together, these facts give that the Jacobian matrix for $h$ (under the basis $e_2, ..., e_n$) has the following block form:
\begin{equation}
   \left(\frac{\partial z}{\partial \tilde{x}}\right) = \left[\begin{BMAT}{cc}{cc}
    I & \mb{0} \\
    A & \frac{\partial h^*(\tilde{x}_{d+1:n})}{\partial \tilde{x}_{d+1:n})}
    \end{BMAT}\right]
\end{equation}
and by the properties of determinants of block matrices we have that
\begin{equation}
    \left|\textnormal{det}\left(\frac{\partial z}{\partial \tilde{x}}\right)\right| =    \left|\textnormal{det}\left(\frac{\partial h^*(\tilde{x}_{d+1:n})}{\partial \tilde{x}_{d+1:n})}\right)\right|
\end{equation}
\end{proof}
Given Lemma 1, all that remains is to evaluate 
\begin{equation}
    \left|\textnormal{det}\left(\frac{\partial h^*(\tilde{x}_{d+1:n})}{\partial \tilde{x}_{d+1:n})}\right)\right|.
\end{equation}
This can again be done by the chain rule, where we use Facts 2-4 to compute the determinant for exponential map, logarithmic map, and parallel transport.
Finally, the Jacobian determinant for the term
\begin{equation}
    \tilde{x} \odot \sigma(s(b \odot \tilde{x})))
\end{equation}
can easily be computed as $\prod_{j=d+1}^n\sigma(s(b \odot \tilde{x}))_j$ since the standard Euclidean basis is a basis for the tangent space at the origin as shown in Appendix \ref{Euclidean_RealNVP_appendix}.
\end{proof}

\cut{\input{decoders.tex}}

\section{Model Architectures and Hyperparameters}
\label{model_arch_and_hyperparams}
In this section we provide more details regarding model architectures and hyperparameters for each experiment in \ref{experiments}. For all hyperbolic models we used a curvature warmup for $10$ epochs which aids in numerical stability \citet{skopek2019mixed}. Specifically, we set $R=11$ and linearly decrease to $R=2$ every epoch after which it is treated as a learnable parameter. 

\xhdr{Structured Density Estimation}
For all VAE models our encoder consists of three linear layers. The first layer maps the input to a hidden space and the other two layers are used to paramaterize the mean and variance of the prior distribution and map samples to the latent space. The decoder for these models is simply a small MLP that consists of two linear layers that map the latent space to the hidden space and then finally back to the observation space. One important distinction between Euclidean models and hyperbolic models is that we use aFor BDP the hidden dim size is $200$ while for MNIST we use $600$ and the latent space is varied as shown in Tables \ref{table:bdp_table} and \ref{table:mnist_table}. All flow models used in this setting consist of $2$ linear layers each of size $128$. Between each layer in either the encoder and decoder we use the LeakyRelu \cite{xu2015empirical} activation function while tanh is used between flow layers. Lastly, we train all models for $80$ epochs with the Adam optimizer with default setting \cite{kingma2014adam}.

\xhdr{Graph Reconstruction}
For graph reconstruction task we use the VGAE model as a base \cite{kipf2016variational} which also uses three linear layers of size $16$ as the encoder in the VAE model. The decoder however is parameter less and is simply an inner product either in Euclidean space or in $\mathcal{T}_{\textbf{o}}\mathbb{H}^n_K$ for Hyperbolic models. As the reconstruction objective contains $N^2$ terms we rescale the $\mathbb{D}_{KL}$ penalty by a factor of $1/N$ such that each of the losses are on the same scale. This can be understood as a $\beta-$VAE like model where $\beta = 1/N$. Like the structured density estimation setting all our flow models consist of two linear layers of size $128$ with a tanh nonlinearity. Finally, we train the each model for $3000$ epochs using the Adam optimizer \cite{kingma2014adam}. 

\xhdr{Graph Generation}
For the graph generation task we adapt the training setup from \cite{liu2019graph} in that we pretrain a graph autoencoder for $100$ epochs to generate node latents. Empirically, we found that using a VGAE model for hyperbolic space worked better than a vanilla a GAE model. Furthermore, instead of using simple linear layers for the encoder we use GAT \cite{velivckovic2017graph} layer of size $32$, which has access to the adjacency matrix. We use LeakyReLU for our encoder non-linearity while tanh is used for all flow models. Unlike GRevNets that use node features sampled from $\mathcal{N}(0,I)$ we find that it is necessary to provide the actual adjacency matrix otherwise training did not succeed. Our decoder defines edge probabilities as $p(A_{u,v}=1|z_u,z_v) = \sigma((-d_{\mathcal{G}}(u,v) - b)/\tau)$ where $b$ and $\tau$ are learned edge specific bias and temperature parameters implemented as one GAT layer followed by a linear layer both of size $32$. Thus both the encoder and decoder are both parameterized and optimized using the Adam optimizer \cite{kingma2014adam}. 

\section{Additional Density Estimation Results}
\label{additional_results}
We now provide additional qualitative results for density estimation in hyperbolic space as visualized in the Poincar\'e disk. For these visualizations we take a density initially defined on Euclidean space and project them to the hyperboloid using the logarithmic map at the origin. We then sample $500$ points from this new density and fit both $\mathcal{T}C$ and $\mathcal{W}\mathbb{H}C$ based flows. The results for the learned densities are shown below in Figure \ref{fig:additional_density_estimation}.

\begin{figure}[ht]
     \centering
     \includegraphics[width=0.95\linewidth]{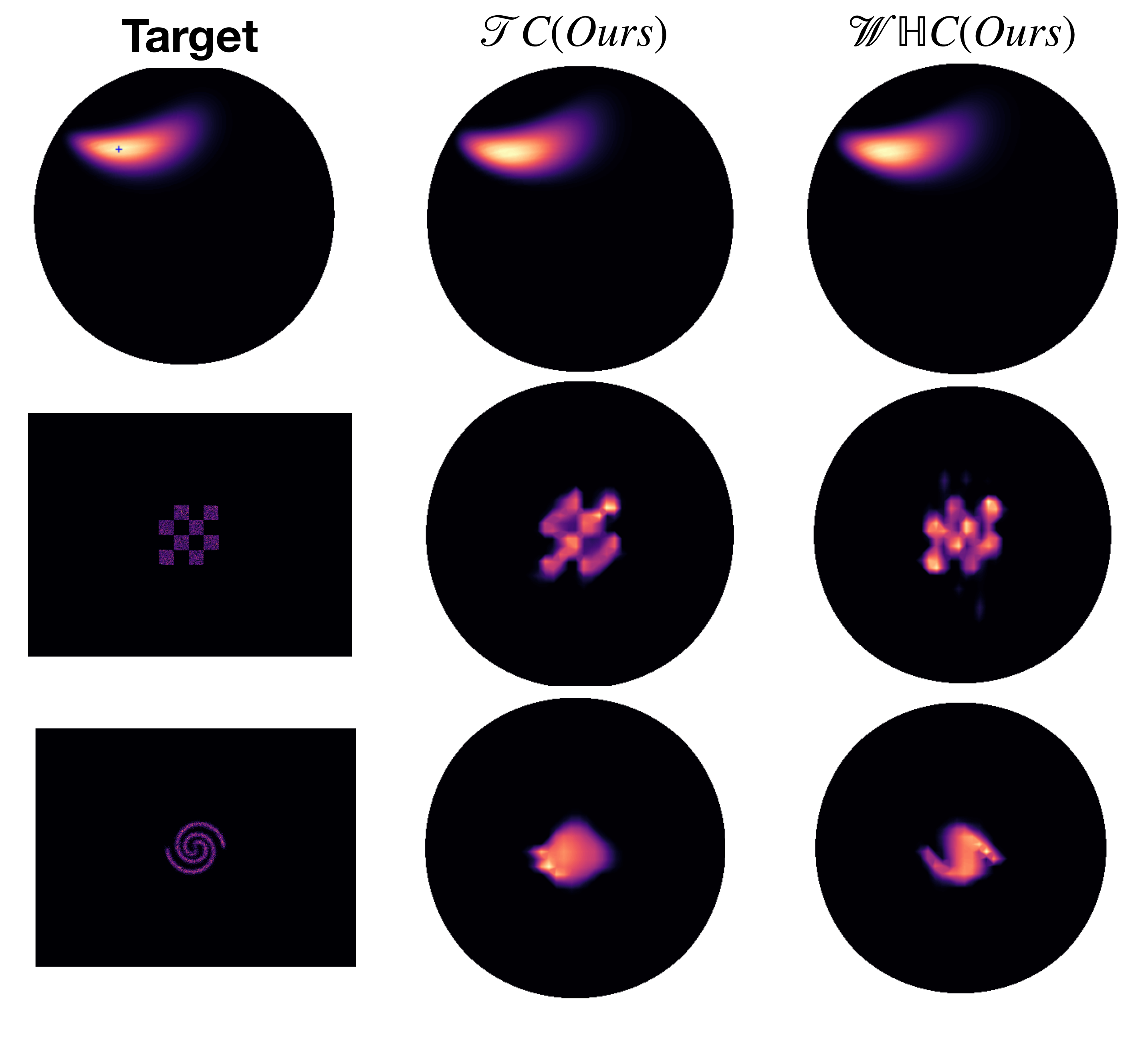}
     \caption{\textbf{Top:} Wrapped Gaussian with $\mu = [-1.0,1.0]$ and $\sigma = [1.0, 0.25]^T$. \textbf{Mid:} Checkerboard pattern projected to hyperbolic space. \textbf{Bot:} 2D Spiral projected to hyperbolic space}
     \label{fig:additional_density_estimation}
 \end{figure}

\section{Dataset Issues}
\label{dataset_issues}
Upon inspecting the code and data kindly provided by \citet{mathieu2019continuous} we uncovered some issues that led to us omitting their CS-PhD and Phylogenetics datasets in our comparisons. 
In particular, \citet{mathieu2019continuous} use a decoder in their cross-entropy loss that does not define a proper probability. 
This appears to have caused optimization issues that artificially deflated the reported performance of all the models investigated in that work.
When substituting in the dot product decoder employed in this work, the accuracy of all models increases dramatically. 
After this change, there is no longer any benefit from employing hyperbolic spaces on these datasets. 
In particular, after applying this fix, the performance of the hyperbolic VAE used by \citet{mathieu2019continuous} falls substantially below a Euclidean VAE.
Since we expect our hyperbolic flows to only give gains in cases where hyperbolic spaces provide a benefit over Euclidean spaces, these datasets do not provide a meaningful testbed for our proposed approach. 
Lastly, upon inspecting the code and data in \citet{mathieu2019continuous}, we also found that the columns 1 and 2 in Table 4 of their paper appear to be swapped compared to the results generated by their code. 

%% file: decoders.tex
\section{Decoders}
\noindent
\textbf{Issue:} Our VGAE models are trained using a cross-entropy loss over the predictions
\begin{equation}
    \hat{P}(A[i,j] = 1 | \Theta),
\end{equation}
where $\Theta$ is our encoder/decoder model parameters. However, many of our decoders are not giving proper probabilities, 
for example taking a sigmoid of distances between node latents
\begin{equation}
    \hat{P}(A[i,j] = 1 | \Theta) = \sigma(-d(\mb{z}_i, \mb{z}_j))
\end{equation}
is improper since $\sigma(-d(\mb{z}_i, \mb{z}_j)) \in [0,0.5]$ and is not surjective onto $[0,1]$. (The root cause is the fact that $d(\mb{z}_i, \mb{z}_j) \geq 0$.)
{\bf Note that this is an issue at both training and inference, as our cross-entropy loss during training will also struggle to learn when using improper probability inputs.}

Note that the dot-product decoder does not have this problem because 
\begin{equation}
  \hat{P}(A[i,j] = 1 | \Theta) = \sigma(\mb{z}_i^\top \mb{z}_j) 
\end{equation}
is surjective on $[0,1]$ since $\mb{z}_i^\top \mb{z}_j\in [-\inf, \inf]$.
{\bf However, we want to use a distance-based decoder, because that allows us to leverage the structure of the hyperbolic space in the most proper/elegant/intuitive way.}\\

\noindent
\textbf{Direct Solution} We use a well-defined softmax. The most direct solution to give proper probabilties for edges within a single graph based on distances would be:
\begin{equation}
    \hat{P}(A[i,j] = 1 | \Theta) = \sigma(p)\times\frac{\exp(d(\mb{z}_i, \mb{z}_j))}{\sum_{k=1}^{n}\sum_{l=l+1}^{n}\exp(d(\mb{z}_k, \mb{z}_l))} 
\end{equation}
where $p \in \mathbb{R}$ is a learnable parameter that controls the density of the graph.  
Note that here we are normalizing over all possible edges in the graph. 
Thus, a pair of nodes that are close together {\em relative to the average pair of nodes in the graph}, will tend to be connected. 
The $\sigma(p)$ term controls the density of the graph. 
{\bf This equation would be perfect if we were only learning a model over a single graph... However, if we are training over a distribution of multiple graphs that might have different numbers of nodes then the probabilities will all be lower in larger graphs, which is no good.}
In particular, if all the nodes are equidistant then we will get lower probabilities for a graph with more nodes...
A natural way to fix this is as follows:
\begin{equation}
    \hat{P}^{\textrm{fixed}}(A[i,j] = 1 | \Theta) = \sigma(p)\times\frac{\exp(d(\mb{z}_i, \mb{z}_j))}{\sum_{k=1}^{n}\sum_{l=l+1}^{n}\exp(d(\mb{z}_k, \mb{z}_l))} \times \left(\max_{i',j'}\left\{\frac{\exp(d(\mb{z}_{i'}, \mb{z}_{j'}))}{\sum_{k=1}^{n}\sum_{l=l+1}^{n}\exp(d(\mb{z}_k, \mb{z}_l))}\right\}\right)^{-1}
\end{equation}
That is, we normalize the softmax score by the inverse of the maximum softmax value for a graph.
This ensures that the decoder simplifies to the Erdos-Renyi model (with edge probability $\sigma(p)$) when all pairwise distances are equal (in expectation), regardless of how many nodes are in the graph.
Note that the $\sigma(p)$ term is important since it is what allows the model to learn the density of the graph.

%% file: main.bbl
\begin{thebibliography}{53}
\providecommand{\natexlab}[1]{#1}
\providecommand{\url}[1]{\texttt{#1}}
\expandafter\ifx\csname urlstyle\endcsname\relax
  \providecommand{\doi}[1]{doi: #1}\else
  \providecommand{\doi}{doi: \begingroup \urlstyle{rm}\Url}\fi

\bibitem[Anderson et~al.(1992)Anderson, Anderson, and
  May]{anderson1992infectious}
Anderson, R.~M., Anderson, B., and May, R.~M.
\newblock \emph{Infectious diseases of humans: dynamics and control}.
\newblock Oxford university press, 1992.

\bibitem[Brehmer \& Cranmer(2020)Brehmer and Cranmer]{brehmer2020flows}
Brehmer, J. and Cranmer, K.
\newblock Flows for simultaneous manifold learning and density estimation.
\newblock \emph{arXiv preprint arXiv:2003.13913}, 2020.

\bibitem[Burda et~al.(2015)Burda, Grosse, and
  Salakhutdinov]{burda2015importance}
Burda, Y., Grosse, R., and Salakhutdinov, R.
\newblock Importance weighted autoencoders.
\newblock \emph{arXiv preprint arXiv:1509.00519}, 2015.

\bibitem[Chamberlain et~al.(2017)Chamberlain, Clough, and
  Deisenroth]{chamberlain2017neural}
Chamberlain, B.~P., Clough, J., and Deisenroth, M.~P.
\newblock Neural embeddings of graphs in hyperbolic space.
\newblock \emph{arXiv preprint arXiv:1705.10359}, 2017.

\bibitem[Chami et~al.(2019)Chami, Ying, R{\'e}, and
  Leskovec]{chami2019hyperbolic}
Chami, I., Ying, Z., R{\'e}, C., and Leskovec, J.
\newblock Hyperbolic graph convolutional neural networks.
\newblock In \emph{Advances in Neural Information Processing Systems}, pp.\
  4869--4880, 2019.

\bibitem[Davidson et~al.(2018)Davidson, Falorsi, De~Cao, Kipf, and
  Tomczak]{davidson2018hyperspherical}
Davidson, T.~R., Falorsi, L., De~Cao, N., Kipf, T., and Tomczak, J.~M.
\newblock Hyperspherical variational auto-encoders.
\newblock \emph{arXiv preprint arXiv:1804.00891}, 2018.

\bibitem[De~Cao et~al.(2019)De~Cao, Titov, and Aziz]{de2019block}
De~Cao, N., Titov, I., and Aziz, W.
\newblock Block neural autoregressive flow.
\newblock \emph{arXiv preprint arXiv:1904.04676}, 2019.

\bibitem[Dhingra et~al.(2018)Dhingra, Shallue, Norouzi, Dai, and
  Dahl]{dhingra2018embedding}
Dhingra, B., Shallue, C.~J., Norouzi, M., Dai, A.~M., and Dahl, G.~E.
\newblock Embedding text in hyperbolic spaces.
\newblock \emph{arXiv preprint arXiv:1806.04313}, 2018.

\bibitem[Dinh et~al.(2017)Dinh, Sohl-Dickstein, and Bengio]{dinh2016density}
Dinh, L., Sohl-Dickstein, J., and Bengio, S.
\newblock Density estimation using real nvp.
\newblock In \emph{The 5th International Conference on Learning Representations
  (ICLR), Vancouver}, 2017.

\bibitem[Falorsi et~al.(2019)Falorsi, de~Haan, Davidson, and
  Forr{\'e}]{falorsi2019reparameterizing}
Falorsi, L., de~Haan, P., Davidson, T.~R., and Forr{\'e}, P.
\newblock Reparameterizing distributions on lie groups.
\newblock \emph{arXiv preprint arXiv:1903.02958}, 2019.

\bibitem[Ganea et~al.(2018)Ganea, B{\'e}cigneul, and
  Hofmann]{ganea2018hyperbolic}
Ganea, O., B{\'e}cigneul, G., and Hofmann, T.
\newblock Hyperbolic neural networks.
\newblock In \emph{Advances in neural information processing systems}, pp.\
  5345--5355, 2018.

\bibitem[Gemici et~al.(2016)Gemici, Rezende, and
  Mohamed]{gemici2016normalizing}
Gemici, M.~C., Rezende, D., and Mohamed, S.
\newblock Normalizing flows on riemannian manifolds.
\newblock \emph{arXiv preprint arXiv:1611.02304}, 2016.

\bibitem[Goh et~al.(2007)Goh, Cusick, Valle, Childs, Vidal, and
  Barab{\'a}si]{goh2007human}
Goh, K.-I., Cusick, M.~E., Valle, D., Childs, B., Vidal, M., and Barab{\'a}si,
  A.-L.
\newblock The human disease network.
\newblock \emph{Proceedings of the National Academy of Sciences}, 104\penalty0
  (21):\penalty0 8685--8690, 2007.

\bibitem[Golomb(1996)]{golomb1996polyominoes}
Golomb, S.~W.
\newblock \emph{Polyominoes: puzzles, patterns, problems, and packings},
  volume~16.
\newblock Princeton University Press, 1996.

\bibitem[Goodfellow et~al.(2014)Goodfellow, Pouget-Abadie, Mirza, Xu,
  Warde-Farley, Ozair, Courville, and Bengio]{goodfellow2014generative}
Goodfellow, I., Pouget-Abadie, J., Mirza, M., Xu, B., Warde-Farley, D., Ozair,
  S., Courville, A., and Bengio, Y.
\newblock Generative adversarial nets.
\newblock In \emph{Advances in neural information processing systems}, pp.\
  2672--2680, 2014.

\bibitem[Grathwohl et~al.(2018)Grathwohl, Chen, Bettencourt, Sutskever, and
  Duvenaud]{grathwohl2018ffjord}
Grathwohl, W., Chen, R.~T., Bettencourt, J., Sutskever, I., and Duvenaud, D.
\newblock Ffjord: Free-form continuous dynamics for scalable reversible
  generative models.
\newblock \emph{arXiv preprint arXiv:1810.01367}, 2018.

\bibitem[Grattarola et~al.(2019)Grattarola, Livi, and
  Alippi]{grattarola2019adversarial}
Grattarola, D., Livi, L., and Alippi, C.
\newblock Adversarial autoencoders with constant-curvature latent manifolds.
\newblock \emph{Applied Soft Computing}, 81:\penalty0 105511, 2019.

\bibitem[Gulcehre et~al.(2018)Gulcehre, Denil, Malinowski, Razavi, Pascanu,
  Hermann, Battaglia, Bapst, Raposo, Santoro, et~al.]{gulcehre2018hyperbolic}
Gulcehre, C., Denil, M., Malinowski, M., Razavi, A., Pascanu, R., Hermann,
  K.~M., Battaglia, P., Bapst, V., Raposo, D., Santoro, A., et~al.
\newblock Hyperbolic attention networks.
\newblock \emph{arXiv preprint arXiv:1805.09786}, 2018.

\bibitem[Hoffman et~al.(2013)Hoffman, Blei, Wang, and
  Paisley]{hoffman2013stochastic}
Hoffman, M.~D., Blei, D.~M., Wang, C., and Paisley, J.
\newblock Stochastic variational inference.
\newblock \emph{The Journal of Machine Learning Research}, 14\penalty0
  (1):\penalty0 1303--1347, 2013.

\bibitem[Huang et~al.(2018)Huang, Krueger, Lacoste, and
  Courville]{huang2018neural}
Huang, C.-W., Krueger, D., Lacoste, A., and Courville, A.
\newblock Neural autoregressive flows.
\newblock In \emph{Proceedings of the 35th international conference on Machine
  learning}, 2018.

\bibitem[Khrulkov et~al.(2019)Khrulkov, Mirvakhabova, Ustinova, Oseledets, and
  Lempitsky]{khrulkov2019hyperbolic}
Khrulkov, V., Mirvakhabova, L., Ustinova, E., Oseledets, I., and Lempitsky, V.
\newblock Hyperbolic image embeddings.
\newblock \emph{arXiv preprint arXiv:1904.02239}, 2019.

\bibitem[Kingma \& Ba(2014)Kingma and Ba]{kingma2014adam}
Kingma, D.~P. and Ba, J.
\newblock Adam: A method for stochastic optimization.
\newblock \emph{arXiv preprint arXiv:1412.6980}, 2014.

\bibitem[Kingma \& Welling(2013)Kingma and Welling]{kingma2013auto}
Kingma, D.~P. and Welling, M.
\newblock Auto-encoding variational bayes.
\newblock \emph{arXiv preprint arXiv:1312.6114}, 2013.

\bibitem[Kipf \& Welling(2016)Kipf and Welling]{kipf2016variational}
Kipf, T.~N. and Welling, M.
\newblock Variational graph auto-encoders.
\newblock \emph{arXiv preprint arXiv:1611.07308}, 2016.

\bibitem[Kobyzev et~al.(2019)Kobyzev, Prince, and
  Brubaker]{kobyzev2019normalizing}
Kobyzev, I., Prince, S., and Brubaker, M.~A.
\newblock Normalizing flows: Introduction and ideas.
\newblock \emph{arXiv preprint arXiv:1908.09257}, 2019.

\bibitem[Law et~al.(2019)Law, Liao, Snell, and Zemel]{law2019lorentzian}
Law, M., Liao, R., Snell, J., and Zemel, R.
\newblock Lorentzian distance learning for hyperbolic representations.
\newblock In \emph{International Conference on Machine Learning}, pp.\
  3672--3681, 2019.

\bibitem[Li \& Turner(2017)Li and Turner]{li2017gradient}
Li, Y. and Turner, R.~E.
\newblock Gradient estimators for implicit models.
\newblock \emph{arXiv preprint arXiv:1705.07107}, 2017.

\bibitem[Liao et~al.(2019)Liao, Li, Song, Wang, Hamilton, Duvenaud, Urtasun,
  and Zemel]{liao2019efficient}
Liao, R., Li, Y., Song, Y., Wang, S., Hamilton, W., Duvenaud, D.~K., Urtasun,
  R., and Zemel, R.
\newblock Efficient graph generation with graph recurrent attention networks.
\newblock In \emph{Advances in Neural Information Processing Systems}, pp.\
  4257--4267, 2019.

\bibitem[Liu et~al.(2019{\natexlab{a}})Liu, Kumar, Ba, Kiros, and
  Swersky]{liu2019graph}
Liu, J., Kumar, A., Ba, J., Kiros, J., and Swersky, K.
\newblock Graph normalizing flows.
\newblock In \emph{Advances in Neural Information Processing Systems}, pp.\
  13556--13566, 2019{\natexlab{a}}.

\bibitem[Liu et~al.(2019{\natexlab{b}})Liu, Nickel, and
  Kiela]{liu2019hyperbolic}
Liu, Q., Nickel, M., and Kiela, D.
\newblock Hyperbolic graph neural networks.
\newblock In \emph{Advances in Neural Information Processing Systems}, pp.\
  8228--8239, 2019{\natexlab{b}}.

\bibitem[Makhzani et~al.(2015)Makhzani, Shlens, Jaitly, Goodfellow, and
  Frey]{makhzani2015adversarial}
Makhzani, A., Shlens, J., Jaitly, N., Goodfellow, I., and Frey, B.
\newblock Adversarial autoencoders.
\newblock \emph{arXiv preprint arXiv:1511.05644}, 2015.

\bibitem[Mathieu et~al.(2019)Mathieu, Le~Lan, Maddison, Tomioka, and
  Teh]{mathieu2019continuous}
Mathieu, E., Le~Lan, C., Maddison, C.~J., Tomioka, R., and Teh, Y.~W.
\newblock Continuous hierarchical representations with poincar{\'e} variational
  auto-encoders.
\newblock In \emph{Advances in neural information processing systems}, pp.\
  12544--12555, 2019.

\bibitem[Nagano et~al.(2019)Nagano, Yamaguchi, Fujita, and
  Koyama]{nagano2019wrapped}
Nagano, Y., Yamaguchi, S., Fujita, Y., and Koyama, M.
\newblock A wrapped normal distribution on hyperbolic space for gradient-based
  learning.
\newblock In \emph{International Conference on Machine Learning}, pp.\
  4693--4702, 2019.

\bibitem[Nickel \& Kiela(2017)Nickel and Kiela]{nickel2017poincare}
Nickel, M. and Kiela, D.
\newblock Poincar{\'e} embeddings for learning hierarchical representations.
\newblock In \emph{Advances in neural information processing systems}, pp.\
  6338--6347, 2017.

\bibitem[Nickel \& Kiela(2018)Nickel and Kiela]{nickel2018learning}
Nickel, M. and Kiela, D.
\newblock Learning continuous hierarchies in the lorentz model of hyperbolic
  geometry.
\newblock \emph{arXiv preprint arXiv:1806.03417}, 2018.

\bibitem[Ovinnikov(2019)]{ovinnikov2019poincar}
Ovinnikov, I.
\newblock Poincar$\backslash$'e wasserstein autoencoder.
\newblock \emph{arXiv preprint arXiv:1901.01427}, 2019.

\bibitem[Papamakarios et~al.(2019)Papamakarios, Nalisnick, Rezende, Mohamed,
  and Lakshminarayanan]{papamakarios2019normalizing}
Papamakarios, G., Nalisnick, E., Rezende, D.~J., Mohamed, S., and
  Lakshminarayanan, B.
\newblock Normalizing flows for probabilistic modeling and inference.
\newblock \emph{arXiv preprint arXiv:1912.02762}, 2019.

\bibitem[Pennec(2006)]{pennec2006intrinsic}
Pennec, X.
\newblock Intrinsic statistics on riemannian manifolds: Basic tools for
  geometric measurements.
\newblock \emph{Journal of Mathematical Imaging and Vision}, 25\penalty0
  (1):\penalty0 127, 2006.

\bibitem[Ratcliffe(1994)]{Ratcliffe94}
Ratcliffe, J.~G.
\newblock \emph{Foundations of Hyperbolic Manifolds}.
\newblock Number 149 in Graduate Texts in Mathematics. Springer-Verlag, 1994.

\bibitem[Rezende \& Mohamed(2015)Rezende and Mohamed]{rezende2015variational}
Rezende, D.~J. and Mohamed, S.
\newblock Variational inference with normalizing flows.
\newblock In \emph{Proceedings of the 32nd international conference on Machine
  learning}. ACM, 2015.

\bibitem[Rezende et~al.(2020)Rezende, Papamakarios, Racani{\`e}re, Albergo,
  Kanwar, Shanahan, and Cranmer]{rezende2020normalizing}
Rezende, D.~J., Papamakarios, G., Racani{\`e}re, S., Albergo, M.~S., Kanwar,
  G., Shanahan, P.~E., and Cranmer, K.
\newblock Normalizing flows on tori and spheres.
\newblock \emph{arXiv preprint arXiv:2002.02428}, 2020.

\bibitem[Said et~al.(2014)Said, Bombrun, and Berthoumieu]{said2014new}
Said, S., Bombrun, L., and Berthoumieu, Y.
\newblock New riemannian priors on the univariate normal model.
\newblock \emph{Entropy}, 16\penalty0 (7):\penalty0 4015--4031, 2014.

\bibitem[Sarkar(2011)]{sarkar2011low}
Sarkar, R.
\newblock Low distortion delaunay embedding of trees in hyperbolic plane.
\newblock In \emph{International Symposium on Graph Drawing}, pp.\  355--366.
  Springer, 2011.

\bibitem[Scarselli et~al.(2008)Scarselli, Gori, Tsoi, Hagenbuchner, and
  Monfardini]{scarselli2008graph}
Scarselli, F., Gori, M., Tsoi, A.~C., Hagenbuchner, M., and Monfardini, G.
\newblock The graph neural network model.
\newblock \emph{IEEE Transactions on Neural Networks}, 20\penalty0
  (1):\penalty0 61--80, 2008.

\bibitem[Shi et~al.(2017)Shi, Sun, and Zhu]{shi2017kernel}
Shi, J., Sun, S., and Zhu, J.
\newblock Kernel implicit variational inference.
\newblock \emph{arXiv preprint arXiv:1705.10119}, 2017.

\bibitem[Shi et~al.(2018)Shi, Sun, and Zhu]{shi2018spectral}
Shi, J., Sun, S., and Zhu, J.
\newblock A spectral approach to gradient estimation for implicit
  distributions.
\newblock \emph{arXiv preprint arXiv:1806.02925}, 2018.

\bibitem[Skopek et~al.(2019)Skopek, Ganea, and B{\'e}cigneul]{skopek2019mixed}
Skopek, O., Ganea, O.-E., and B{\'e}cigneul, G.
\newblock Mixed-curvature variational autoencoders.
\newblock \emph{arXiv preprint arXiv:1911.08411}, 2019.

\bibitem[Tay et~al.(2018)Tay, Tuan, and Hui]{tay2018hyperbolic}
Tay, Y., Tuan, L.~A., and Hui, S.~C.
\newblock Hyperbolic representation learning for fast and efficient neural
  question answering.
\newblock In \emph{Proceedings of the Eleventh ACM International Conference on
  Web Search and Data Mining}, pp.\  583--591, 2018.

\bibitem[Tolstikhin et~al.(2017)Tolstikhin, Bousquet, Gelly, and
  Schoelkopf]{tolstikhin2017wasserstein}
Tolstikhin, I., Bousquet, O., Gelly, S., and Schoelkopf, B.
\newblock Wasserstein auto-encoders.
\newblock \emph{arXiv preprint arXiv:1711.01558}, 2017.

\bibitem[Veli{\v{c}}kovi{\'c} et~al.(2017)Veli{\v{c}}kovi{\'c}, Cucurull,
  Casanova, Romero, Lio, and Bengio]{velivckovic2017graph}
Veli{\v{c}}kovi{\'c}, P., Cucurull, G., Casanova, A., Romero, A., Lio, P., and
  Bengio, Y.
\newblock Graph attention networks.
\newblock \emph{arXiv preprint arXiv:1710.10903}, 2017.

\bibitem[Xu et~al.(2015)Xu, Wang, Chen, and Li]{xu2015empirical}
Xu, B., Wang, N., Chen, T., and Li, M.
\newblock Empirical evaluation of rectified activations in convolutional
  network.
\newblock \emph{arXiv preprint arXiv:1505.00853}, 2015.

\bibitem[Yang et~al.(2016)Yang, Cohen, and Salakhutdinov]{yang2016revisiting}
Yang, Z., Cohen, W.~W., and Salakhutdinov, R.
\newblock Revisiting semi-supervised learning with graph embeddings.
\newblock \emph{arXiv preprint arXiv:1603.08861}, 2016.

\bibitem[Yu \& De~Sa(2019)Yu and De~Sa]{yu2019numerically}
Yu, T. and De~Sa, C.~M.
\newblock Numerically accurate hyperbolic embeddings using tiling-based models.
\newblock In \emph{Advances in Neural Information Processing Systems}, pp.\
  2021--2031, 2019.

\end{thebibliography}
